\documentclass[letterpaper]{article} 
\usepackage[preprint]{aaai2027}  
\usepackage[hyphens]{url}  
\usepackage{graphicx} 
\def\UrlFont{\rm}  
\usepackage{natbib}  
\usepackage{caption} 
\usepackage{algorithm}
\usepackage{algorithmic}

\usepackage{newfloat}
\usepackage{listings}
\DeclareCaptionStyle{ruled}{labelfont=normalfont,labelsep=colon,strut=off} 
\floatstyle{ruled}
\newfloat{listing}{tb}{lst}{}
\floatname{listing}{Listing}

\usepackage{booktabs}

\usepackage{xcolor}         
\usepackage{colortbl}       
\usepackage{multirow}       
\usepackage{placeins}
\usepackage{amsthm}
\usepackage{amsmath}
\usepackage{amssymb}

\newtheorem{proposition}{Proposition}

\usepackage{arydshln}

\title{Making Every Step Count: Spatio-Temporal Information Allocation for Imaging Inverse Problems}
\author{
    Yi Cao\textsuperscript{\rm 1},
    Xiangyong Cao\textsuperscript{\rm 1}\corresponding,
    Pei Liu\textsuperscript{\rm 1},
    Yong-Jin Liu\textsuperscript{\rm 2},
    Deyu Meng\textsuperscript{\rm 1}
}
\affiliations{
    \textsuperscript{\rm 1}Xi'an Jiaotong University\\
    \textsuperscript{\rm 2}Tsinghua University
}

\begin{document}

\maketitle

\begin{abstract}
Flow-based generative models have emerged as powerful image priors for training-free inverse problem solving, capturing coherent semantics and fine-grained structure. Despite these strengths, existing flow-based inverse solvers primarily focus on the design of individual updates, largely overlooking spatio-temporal information allocation under a fixed number of function evaluations (NFEs). Temporally, insufficient early exploration can trap the flow trajectory in an incorrect semantic basin, whereas excessive allocation of NFEs to early stages leaves little budget for late-stage refinement. Spatially, data consistency provides direct constraints only within observed regions, whereas the recovery of missing regions relies mainly on the generative prior. To address these two issues, we introduce two complementary and training-free components, i.e., Spectrum-Adaptive Scheduling (SAS) and Measurement-Prioritized Attention (MPA). For temporal allocation, SAS distributes the available NFEs over flow time according to the degradation spectrum and logSNR geometry, thus better balancing semantic exploration and detail refinement. For spatial propagation, MPA exploits data-prior conflicts to guide information toward weakly constrained regions, thereby enhancing semantic and structural fidelity. Extensive experiments on standard image inverse problems, e.g., super-resolution, motion deblurring, and inpainting, demonstrate that the proposed components can be integrated into existing flow-based inverse solvers in a plug-and-play manner without retraining or additional flow-model evaluations, and can also significantly improve the restoration quality of existing solvers. 
\end{abstract}


\section{Introduction}

\begin{figure}[!t]
    \centering
    \includegraphics[width=0.85\linewidth]{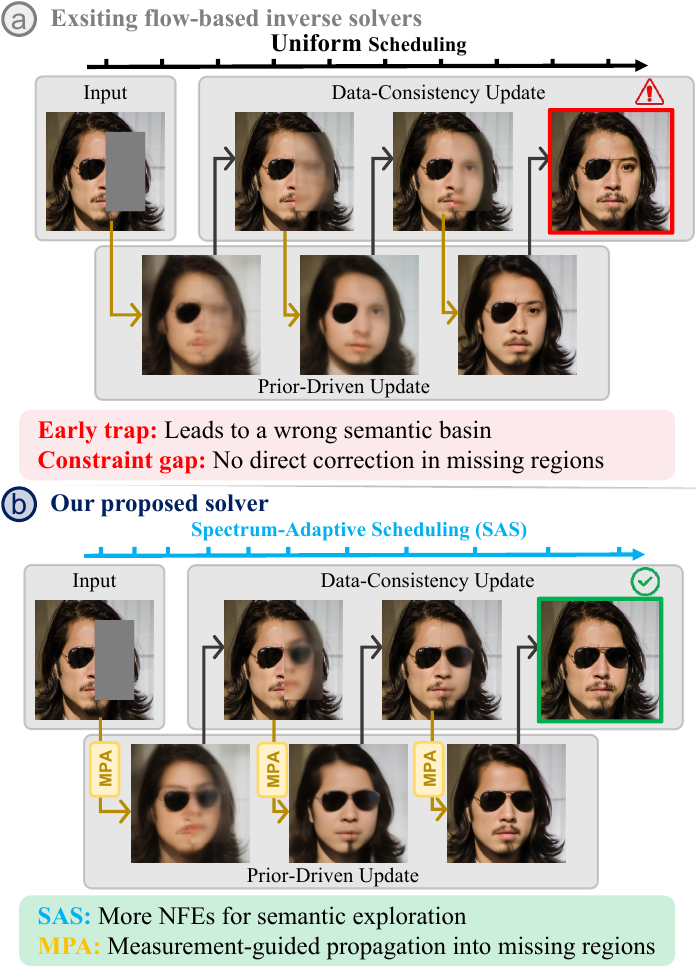}
    \caption{Comparison of existing solvers and our method.
    (a) Existing solvers adopt uniform scheduling, potentially limiting early exploration, while data-consistency updates constrain only observed regions, leaving missing regions largely dependent on the generative prior. (b) Our method contains two main components, i.e., SAS and MPA. SAS allocates NFEs based on operator demand to balance semantic exploration and detail refinement, while MPA strengthens measurement guidance in weakly constrained regions.}
    \label{fig:motivation-overview}
\end{figure}

Flow-based generative
models~\cite{lipman2023flow,liu2023flow,albergo2023stochastic}, which underpin
large-scale text-to-image systems such as Stable
Diffusion~3~\cite{esser2024sd3}, have
demonstrated remarkable image-generation quality and consequently emerged as
powerful priors for ill-posed imaging inverse problems. Accordingly, recent training-free solvers combine a pretrained flow-matching
model with a known measurement operator at inference time, avoiding
task-specific retraining for super-resolution, deblurring, or inpainting.
Within this paradigm, existing methods employ diverse strategies, including
posterior guidance, variational optimization, Langevin--proximal refinement,
and feasible-set projection
~\cite{kim2025flowdps,erbach2025flair,park2026flowlps,pourya2025flower}.
Collectively, these approaches have achieved strong restoration performance
across a range of inverse problems. Nevertheless, under a finite NFE budget,
their designs still center on how each individual update is performed,
while much less attention has been paid to their temporal placement along the
flow trajectory. The spatial propagation of measurement information is
similarly overlooked: data consistency directly constrains only observed
regions, leaving missing regions predominantly dependent on the generative
prior.

In practice, existing solvers typically adopt a uniform timestep schedule at a
fixed NFE budget, treating stages from semantic formation to detail refinement
equally. Yet different degradation operators require different balances of
these stages.
Inpainting provides the clearest example, as illustrated in
Figure~\ref{fig:motivation-overview}(a). Insufficient early-stage
exploration can trap the trajectory in an incorrect semantic basin, causing
structures such as sunglass frames to be omitted from the completion.
Meanwhile, data consistency directly constrains only observed pixels and offers
no direct correction inside the missing region, which therefore remains
predominantly dependent on the generative prior. Other operators also exhibit
distinct information deficits. High-factor super-resolution leaves
high-frequency details unobserved, whereas motion deblurring preserves coarse
semantics but weakly observes many modes. These operator-specific deficits
place different demands on how a finite NFE budget should be allocated, while
unobserved or weakly observed details receive little direct correction from
data consistency.


These observations suggest that flow inversion under a finite NFE budget is not only an update-design problem, but also a problem of temporal budget allocation and spatial information propagation. \textit{Temporally}, the goal is to allocate the NFE budget based on degradation-specific demand, balancing semantic exploration and detail refinement; \textit{spatially}, it is to
strengthen measurement guidance for recovery in weakly constrained regions. To address these challenges, we introduce Spectrum-Adaptive Scheduling (SAS) for NFE budget allocation and Measurement-Prioritized Attention (MPA) for spatial information propagation. Figure~\ref{fig:motivation-overview}(b) illustrates the complementary design, which improves restoration quality without retraining or additional flow-model evaluations.

SAS determines how a
finite NFE budget should be distributed along flow time according to the
degradation spectrum and logSNR geometry. The spectrum distinguishes missing
modes that rely more strongly on prior-driven synthesis from attenuated modes
that retain partial measurement support. LogSNR characterizes how the demands for semantic
exploration and detail refinement change along flow time. Based on the
degradation spectrum and logSNR geometry, SAS defines an operator demand and
constructs an operator-specific schedule, placing more timesteps in high-demand
portions of the trajectory while keeping the total NFE budget unchanged.
Meanwhile, MPA directly modulates image-token self-attention within the DiT
backbone to strengthen measurement guidance in weakly constrained regions. This design builds on the observation that the prior--data discrepancy
highlights structures poorly captured by the current prior. MPA summarizes this discrepancy as a conflict heatmap and transforms it through conflict-guided bias generation into an image-token attention bias, which reshapes the attention weights to improve structural fidelity in weakly constrained regions.

Our contributions are as follows:
\begin{itemize}
    \item We identify a shared bottleneck in training-free flow inverse solvers and formulate it in terms of the temporal allocation of a finite NFE budget and the spatial information propagation. 

    \item We introduce SAS and MPA to address the temporal and spatial
    dimensions, respectively. SAS distributes the NFE budget
    over flow time according to the degradation spectrum and logSNR geometry,
    while MPA uses prior--data conflicts to strengthen measurement guidance in
    weakly constrained regions.

    \item Both components can be integrated into existing solvers without
    retraining or additional flow-model evaluations. Experiments on
    super-resolution, motion deblurring, and inpainting demonstrate consistent improvements in
    restoration quality and better preservation of instance-specific semantic
    and structural details.
\end{itemize}

\section{Background and Related Work}
\label{sec:related}

\subsection{Flow Priors for Inverse Problems}
We consider an imaging inverse problem
\begin{equation}
    y=\mathcal{A}x+\nu ,
\end{equation}
where $\mathcal{A}$ is the degradation operator and $\nu$ is measurement
noise. Because $\mathcal A$ is generally non-invertible or ill-conditioned, recovery requires a natural-image prior in addition to data fidelity. Diffusion-based methods have established pretrained generative models as broadly applicable priors: RePaint propagates observed-region constraints through resampling, DPS injects likelihood gradients into reverse diffusion, DDRM and DDNM exploit the spectrum or null space of a linear operator, and DAPS decouples noise annealing to make early errors easier to correct~\cite{lugmayr2022repaint,chung2023dps,kawar2022ddrm,wang2023ddnm,zhang2024daps}.

With recent advances in flow matching, pretrained flow models have also been
adopted as generative priors for inverse problems. A common affine probability
path is
\begin{equation}
    x_t=a_t x_0+b_t\epsilon,
    \qquad
    \epsilon\sim\mathcal N(0,I),
    \label{eq:affine_path}
\end{equation}
whose signal-to-noise state is commonly parameterized by the log
signal-to-noise ratio (logSNR)~\cite{kingma2021variational}:
\begin{equation}
    \ell(t)=\log\frac{a_t^2}{b_t^2}.
    \label{eq:general_logsnr}
\end{equation}
For the linear path used in this work, $a_t=1-t$ and $b_t=t$. A velocity
prediction $v_\theta(x_t,t)$ then induces
\begin{equation}
    \hat x_{0|t}=x_t-tv_\theta(x_t,t),\qquad
    \hat\epsilon_{|t}=x_t+(1-t)v_\theta(x_t,t).
    \label{eq:endpoint_predictions}
\end{equation}
These endpoint estimates provide convenient interfaces for likelihood,
projection, or proximal corrections while retaining the pretrained flow
trajectory. FlowDPS applies posterior guidance through these endpoint
predictions~\cite{kim2025flowdps}; FLAIR alternates a variational flow
regularizer with hard data consistency~\cite{erbach2025flair}; FlowLPS
combines Langevin exploration with proximal
refinement~\cite{park2026flowlps}; and FLOWER alternates flow-consistent
denoising with measurement-aware projection~\cite{pourya2025flower}.
Despite their different formulations, these solvers all combine prior-driven
and data-consistency updates at a sequence of flow timesteps. Under a finite
inference budget, however, they largely overlook where these updates should be
placed along flow time and how measurement guidance can be strengthened in
weakly constrained regions.

\subsection{Timestep Scheduling}
Sampling schedules determine how a limited evaluation budget is distributed
along a generative trajectory. EDM shows that noise parameterization and
timestep placement should be designed together, with logSNR providing a
natural coordinate for the progression from noise-dominant synthesis to
clean-dominant refinement~\cite{karras2022edm}. More recent training-free methods design the
grid itself: TORS allocates timesteps according to the geometry of the
sampling trajectory~\cite{zhou2026analyzing}, while Lipschitz-guided schedules
use numerical regularity to resolve difficult
intervals~\cite{chen2025lipschitz}. These schedules are designed primarily
around the generative trajectory or its numerical properties and do not
account for the degradation-specific demands of inverse problems. SAS instead
combines the degradation spectrum with logSNR geometry to construct an
operator-specific schedule that allocates the fixed timestep budget according
to restoration demand.

\subsection{Attention-Based Spatial Guidance}
Attention control offers a complementary means of guiding spatial
information. Prompt-to-Prompt and Attend-and-Excite manipulate
text-conditioned attention, while MasaCtrl and Plug-and-Play Diffusion
Features reuse reference attention or intermediate
features~\cite{hertz2022prompt,chefer2023attend,cao2023masactrl,tumanyan2023plug}.
ControlNet and T2I-Adapter instead learn additional conditioning
modules~\cite{zhang2023adding,mou2024t2i}. These methods rely on prompts,
reference features, or learned conditioning and may require extra
optimization, feature injection, or task-specific modules. Existing inverse
solvers, meanwhile, use measurements primarily for data-fidelity correction,
and rarely use them to guide image-token attention within the generative
backbone. MPA instead converts prior--data conflicts into an image-token
attention bias to strengthen measurement guidance in weakly constrained
regions, without additional flow-model evaluations. Together, SAS and MPA
make more effective use of a finite inference budget through operator-specific
temporal allocation and stronger measurement guidance in weakly constrained
regions.

\section{Methodology}
\label{sec:method}

\begin{figure}
    \centering
    \includegraphics[width=\linewidth]{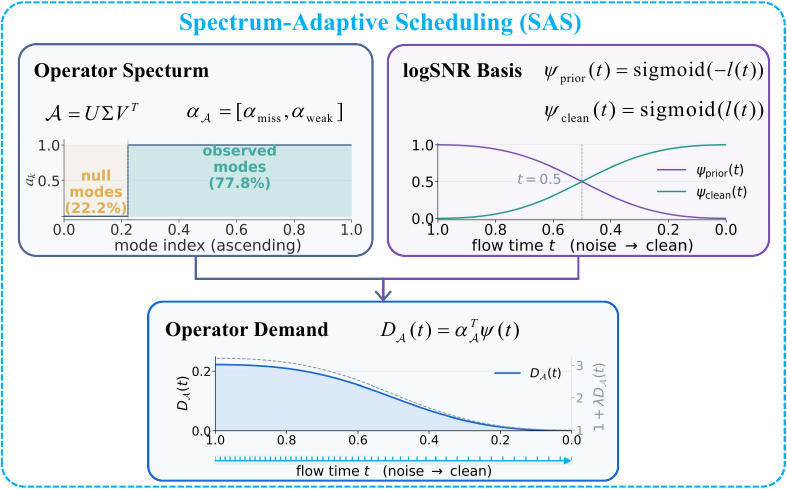}
    \caption{Overview of Spectrum-Adaptive Scheduling (SAS). SAS combines the
    operator spectrum and logSNR-derived temporal bases into an operator demand,
    which is used to distribute the finite NFE budget over flow time.}
    \label{fig:sas-overview}
\end{figure}
\subsection{Overview}
This section details how SAS and MPA address temporal budget allocation and spatial information propagation, respectively, under a finite NFE budget. SAS determines \emph{when} inverse updates are executed by distributing timesteps over flow time according to degradation spectrum and logSNR geometry. MPA strengthens measurement guidance in weakly constrained regions by converting prior--data conflicts into image-token attention bias within the DiT backbone.

\subsection{Spectrum-Adaptive Scheduling}
\label{sec:osta}

As illustrated in Figure.~\ref{fig:sas-overview}, SAS constructs an
operator-specific schedule by quantifying missing and
attenuated spectral modes, mapping their demands over flow time through
logSNR, and allocating the fixed timestep budget by equal-mass quantiles.

\paragraph{Operator spectrum.}
For measurements with $y=\mathcal A x_0+\nu$, the standard quadratic data-consistency objective is
\begin{equation}
    \mathcal L_{\rm dc}(x)
    =
    \frac12\|\mathcal A x-y\|_2^2,
    \label{eq:osta_dc_objective}
\end{equation}
whose Hessian is
\begin{equation}
    \nabla_x^2\mathcal L_{\rm dc}(x)
    =
    \mathcal A^\top\mathcal A.
    \label{eq:osta_dc_hessian}
\end{equation}
Hence, for any unit signal direction $v$, the quadratic form
$v^\top\mathcal A^\top\mathcal A v=\|\mathcal A v\|_2^2$ measures how strongly
it is constrained by the measurement. Moreover, under isotropic Gaussian noise,
$\mathcal A^\top\mathcal A$ is also proportional to the Fisher information
matrix. However, these directions are generally coupled in image coordinates.
Following DDRM~\cite{kawar2022ddrm}, we therefore decouple them via the SVD:
\begin{equation}
    \mathcal A
    =
    U\Sigma V^\top,
    \qquad
    \Sigma
    =
    \operatorname{diag}(a_1,\ldots,a_d),
    \label{eq:osta_svd}
\end{equation}
where $U$ and $V$ provide orthonormal bases for the measurement and image
spaces, respectively, and $\Sigma$ is rectangular diagonal. Let
$r=\operatorname{rank}(\mathcal A)$, with
$a_1\geq\cdots\geq a_r>0=a_{r+1}=\cdots=a_d$. Then
\begin{equation}
    \mathcal A^\top\mathcal A
    =
    V\operatorname{diag}(a_1^2,\ldots,a_d^2)V^\top.
    \label{eq:osta_gram_svd}
\end{equation}
Thus, $v_k$ is an independent signal direction and $a_k^2$ is the
measurement strength along that direction.


We normalize $a_k$ as $\bar a_k=a_k/a_{\max}$, where
$a_{\max}=\|\mathcal A\|_2$, and obtain
\begin{equation}
    G_{\mathcal A}
    :=
    \frac{\mathcal A^\top\mathcal A}{\|\mathcal A\|_2^2}
    =
    V\operatorname{diag}(\bar a_1^2,\ldots,\bar a_d^2)V^\top.
    \label{eq:osta_normalized_gram}
\end{equation}
By construction, each $\bar a_k^2\in[0,1]$ measures how strongly the $k$-th direction is observed relative to the strongest direction.

The ordinary rank $r$ counts all nonzero singular directions equally,
including those that are severely attenuated. We instead use the stable
rank~\cite{tropp2015matrix,cohen2016stable}:
\begin{equation}
    \operatorname{srank}(\mathcal A)
    :=
    \frac{\|\mathcal A\|_F^2}{\|\mathcal A\|_2^2}
    =
    \operatorname{tr}(G_{\mathcal A})
    =
    \sum_{k=1}^d\bar a_k^2.
    \label{eq:osta_stable_rank}
\end{equation}
Stable rank is a scale-invariant effective dimension: a direction contributes
according to its relative measurement strength rather than through a binary
zero/nonzero decision. For every nonzero operator,
\begin{equation}
    1
    \leq
    \operatorname{srank}(\mathcal A)
    \leq
    r
    \leq
    d,
    \label{eq:osta_stable_rank_bounds}
\end{equation}

Relative to full observation, the normalized stable-rank gap separates as
\begin{equation}
    \frac{d-\operatorname{srank}(\mathcal A)}{d}
    =
    \underbrace{\frac{d-r}{d}}_{\alpha_{\rm miss}}
    +
    \underbrace{
    \frac{r-\operatorname{srank}(\mathcal A)}{d}
    }_{\alpha_{\rm weak}},
    \label{eq:spectral_statistics}
\end{equation}
where $\alpha_{\rm miss}$ is the fraction of unobservable directions, while
$\alpha_{\rm weak}$ aggregates attenuation over the observable directions.

\paragraph{LogSNR temporal bases.}
The spectral coefficients above summarize the operator's missing and
attenuated modes. To translate this information into temporal allocation, we
first characterize how the role of an update changes along flow time. At the
noise-dominant early stage, restoration relies more strongly on the prior for
semantic synthesis; toward the clean endpoint, measurement-supported details
can be refined more effectively. We describe this transition using the logSNR
$\ell(t)=\log(a_t^2/b_t^2)$~\cite{kingma2023understanding}, which provides a
natural coordinate for the progression from the noise-dominant stage to the
clean endpoint.

We map logSNR to two complementary temporal bases:
\begin{equation}
    \begin{aligned}
        \psi_{\rm prior}(t)
        &=
        \operatorname{sigmoid}(-\ell(t))
        =
        \frac{b_t^2}{a_t^2+b_t^2},
        \\
        \psi_{\rm clean}(t)
        &=
        \operatorname{sigmoid}(\ell(t))
        =
        \frac{a_t^2}{a_t^2+b_t^2}
    \end{aligned}
    \label{eq:logsnr_bases}
\end{equation}
For the linear flow path $a_t=1-t$ and $b_t=t$, they reduce to
\begin{equation}
    \begin{aligned}
    \psi_{\rm prior}(t)
    &=
    \frac{t^2}{t^2+(1-t)^2},\\
    \psi_{\rm clean}(t)
    &=
    \frac{(1-t)^2}{t^2+(1-t)^2}.
    \end{aligned}
\end{equation}
Here, the two bases are nonnegative and satisfy
$\psi_{\rm prior}(t)+\psi_{\rm clean}(t)=1$:
$\psi_{\rm prior}$ emphasizes early prior-driven exploration, whereas
$\psi_{\rm clean}$ gradually shifts emphasis toward clean-stage detail
refinement.

\paragraph{From operator demand to an adaptive schedule.}
With the degradation spectrum summarized by the operator-spectrum weights and
the restoration demand along flow time characterized by the temporal bases,
we define the operator demand as
\begin{equation}
    D_{\mathcal A}(t)
    :=
    \alpha_{\rm miss}\psi_{\rm prior}(t)
    +
    \alpha_{\rm weak}\psi_{\rm clean}(t).
    \label{eq:osta_profile}
\end{equation}
The pairing reflects the different measurement support of missing and
attenuated modes. Missing modes receive no direct measurement correction and hence predominantly rely on prior-driven synthesis, whereas attenuated modes retain partial measurement support and can benefit more from clean-stage refinement. The resulting $D_{\mathcal A}(t)$ summarizes the operator-specific restoration demand along the flow time and serves as the temporal weighting profile for constructing the adaptive schedule. 

Let $\mathcal T=[t_{\min},t_{\max}]$. We convert this demand directly into a
normalized allocation density:
\begin{equation}
    q_{\mathcal A,\lambda}(t)
    =
    \frac{1+\lambda D_{\mathcal A}(t)}
         {\int_{t_{\min}}^{t_{\max}}
          [1+\lambda D_{\mathcal A}(u)]\,du},
    \qquad \lambda\geq0,
    \label{eq:osta_density}
\end{equation}
where $\lambda$ controls the deviation from uniform allocation, with
$\lambda=0$ recovering the uniform schedule. Given $N$ solver intervals, we
take reverse-time equal-mass quantiles of $q_{\mathcal A,\lambda}$:
\begin{equation}
    \begin{aligned}
        &F_{\mathcal A,\lambda}(t)
    :=
    \int_{t_{\min}}^t q_{\mathcal A,\lambda}(u)\,du,
    \\
    t_i
    &=
    F_{\mathcal A,\lambda}^{-1}
    \left(1-\frac{i}{N}\right),
    \quad i=0,\ldots,N.
    \end{aligned}
    \label{eq:osta_schedule}
\end{equation}
Thus, $t_0=t_{\max}$ and $t_N=t_{\min}$, while portions of the trajectory with
larger operator demand receive denser solver times. Overall, SAS converts the
degradation spectrum and logSNR geometry into an operator-specific schedule,
allowing the fixed timestep budget to be allocated more effectively between
semantic exploration and detail refinement.

\subsection{Measurement-Prioritized Attention}
\label{sec:mgga}

\begin{figure}
    \centering
    \includegraphics[width=\linewidth]{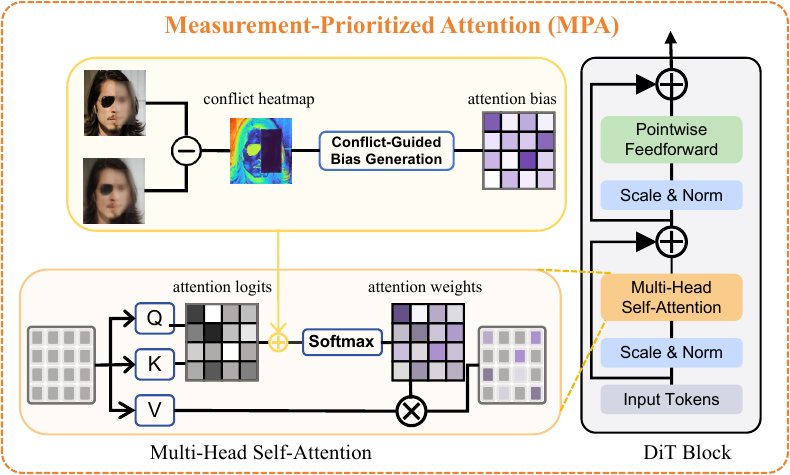}
    \caption{Overview of Measurement-Prioritized Attention (MPA). MPA
    summarizes prior--data conflicts as a conflict heatmap and converts it into
    an image-token attention bias within the DiT backbone, strengthening
    measurement guidance in weakly constrained regions.}
    \label{fig:mpa-overview}
\end{figure}
As illustrated in Figure~\ref{fig:mpa-overview}, MPA strengthens measurement guidance in weakly constrained regions by directly modulating image-token self-attention within the DiT backbone. It first extracts spatial conflicts from the prior--data discrepancy and then converts them into an attention bias that reshapes information flow between image tokens.

\begin{figure}[t]
    \centering
    \includegraphics[width=\linewidth]{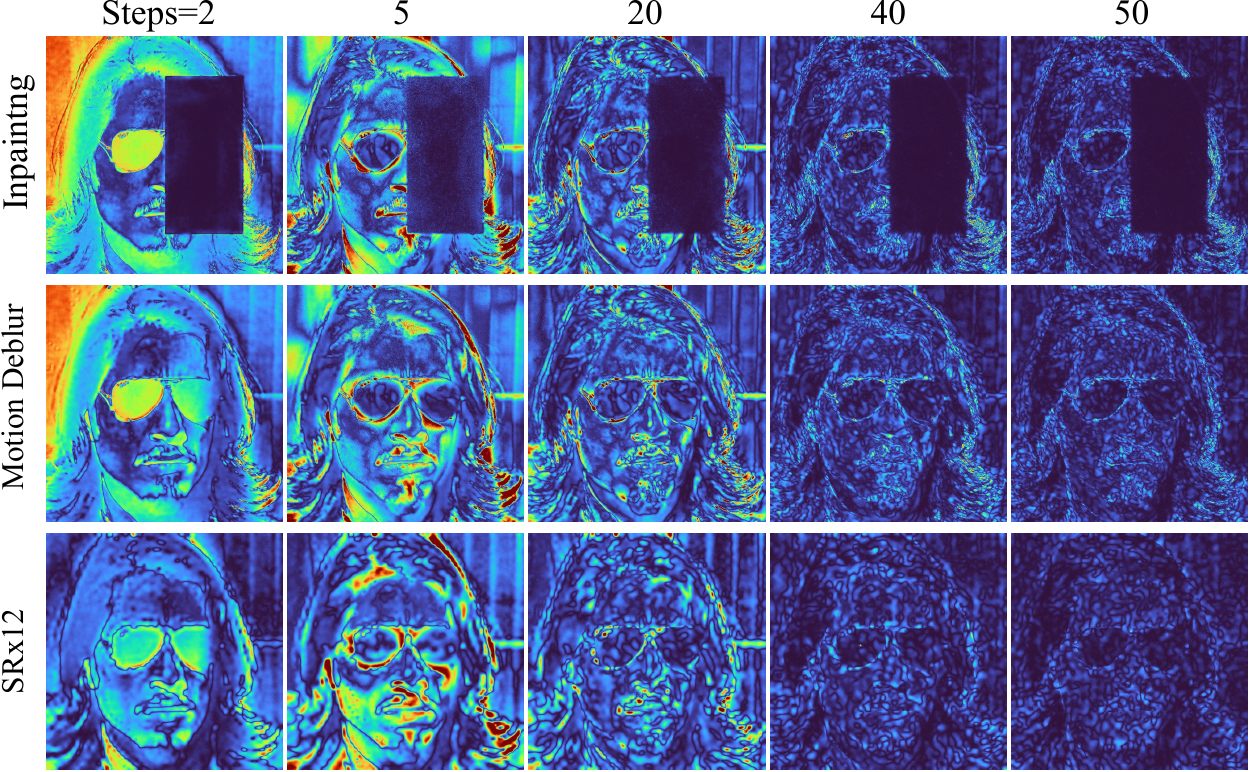}
    \caption{Data-consistency conflict heatmaps over the restoration trajectory
    for inpainting, motion deblurring, and $\times12$ super-resolution. Brighter
    values indicate larger corrections between the prior and data-consistent
    estimates.} 
    \label{fig:conflict_heatmap}
\end{figure}

Let $z_t^{\rm pri}$ and $z_t^{\rm dc}$ denote the latent estimates produced by the prior and data-consistency updates at step $t$, respectively. Their token-wise discrepancy measures the correction introduced by data consistency relative to the current prior estimate and we use its magnitude and summarize it as the conflict heatmap
\begin{equation}
    c_t
    =
    \operatorname{NormPool}
    \left(\left|z_t^{\rm dc}-z_t^{\rm pri}\right|\right)
    \in[0,1]^n,
    \label{eq:conflict_heatmap}
\end{equation}
where $\operatorname{NormPool}$ performs channel pooling, robust normalization, and spatial smoothing. Figure~\ref{fig:conflict_heatmap} visualizes how these
conflicts evolve across solver steps. At early stages, strong responses already
delineate major structural contours and highlight high-frequency details that
the prior tends to suppress. Toward the clean endpoint, the conflicts diminish
and become increasingly concentrated on fine textures. These patterns provide spatial evidence for the subsequent attention modulation.

Suppressing the attention-head index for clarity, let $Q_t,K_t\in\mathbb R^{n\times d_h}$ denote the image-token queries and keys, where $d_h$ is the feature dimension of each head. The original image-token attention logits are
\begin{equation}
    L_t
    =
    \frac{Q_tK_t^\top}{\sqrt{d_h}}
    \in\mathbb R^{n\times n}.
    \label{eq:image_attention_logits}
\end{equation}
The rows of $L_t$ correspond to query tokens and its columns correspond to key tokens. These logits capture feature compatibility between image tokens but do not explicitly account for prior--data conflicts.

To match the pairwise attention logits, MPA broadcasts the token-wise map along the query dimension, giving the attention bias
\begin{equation}
    B_t=\mathbf 1 c_t^\top\in\mathbb R^{n\times n},
    \label{eq:mgga_bias}
\end{equation}
where $\mathbf 1\in\mathbb R^n$ broadcasts $c_t$ across all image-token queries, assigning larger key-side biases to locations with stronger measurement responses. MPA then adds $B_t$ to the original image-token attention logits $L_t$ before softmax:
\begin{equation}
    \widetilde P_t
    =
    \operatorname{Softmax}_{\rm row}
    \left(L_t+\beta_t B_t\right),
    \label{eq:mgga_attention}
\end{equation}
where $\beta_t$ controls the bias strength. The resulting attention weights place greater emphasis on measurement-responsive tokens, thereby strengthening measurement guidance in weakly constrained regions.


\begin{table*}[!t]
\centering
\begingroup
\setlength{\tabcolsep}{2.5pt}
\setlength{\arrayrulewidth}{0.7pt}
\renewcommand{\arraystretch}{1.2}
\arrayrulecolor{black}
\fontsize{8pt}{10pt}\selectfont
\begin{tabular}{l|cccc|cccc|cccc|cccc}
\hline
\multicolumn{1}{c|}{\multirow{2}{*}{\textbf{Method}}}
& \multicolumn{4}{c|}{\textbf{SR $8\times$}}
& \multicolumn{4}{c|}{\textbf{SR $12\times$}}
& \multicolumn{4}{c|}{\textbf{Motion Deblur}}
& \multicolumn{4}{c}{\textbf{Inpainting}} \\
\cline{2-5}\cline{6-9}\cline{10-13}\cline{14-17}
& PSNR$\uparrow$ & SSIM$\uparrow$ & FID$\downarrow$ & LPIPS$\downarrow$
& PSNR$\uparrow$ & SSIM$\uparrow$ & FID$\downarrow$ & LPIPS$\downarrow$
& PSNR$\uparrow$ & SSIM$\uparrow$ & FID$\downarrow$ & LPIPS$\downarrow$
& PSNR$\uparrow$ & SSIM$\uparrow$ & FID$\downarrow$ & LPIPS$\downarrow$ \\
\hline
\multicolumn{17}{c}{\textbf{FFHQ 1k}} \\
\hline
FlowChef
& 27.33 & 0.755 & 41.93 & 0.340
& 26.13 & 0.726 & 58.99 & 0.376
& 27.41 & 0.756 & 36.99 & 0.339
& 18.96 & 0.769 & 65.30 & 0.407 \\
FlowDPS
& 27.52 & 0.702 & 29.13 & 0.419
& 26.62 & 0.691 & 31.54 & 0.453
& 25.93 & 0.681 & 40.58 & 0.485
& 19.98 & 0.761 & 40.84 & 0.336 \\
\hline
FlowLPS
& 24.20 & 0.424 & 53.05 & 0.609
& 23.53 & 0.405 & 65.59 & 0.660
& 33.18 & 0.871 & 31.01 & 0.247
& 21.92 & 0.863 & 17.31 & 0.215 \\
\textbf{FlowLPS + Ours}
& \textbf{26.67} & \textbf{0.611} & \textbf{36.03} & \textbf{0.429}
& \textbf{26.09} & \textbf{0.627} & \textbf{39.281} & \textbf{0.462}
& 33.10 & \textbf{0.872} & \textbf{29.47} & 0.251
& \textbf{24.55} & \textbf{0.868} & 20.86 & 0.230 \\
\hline
FLAIR
& 29.41 & 0.796 & 83.66 & 0.468
& 26.68 & 0.712 & 58.75 & 0.381
& 32.04 & 0.830 & 11.26 & 0.158
& 23.76 & 0.849 & 14.63 & 0.170 \\
\textbf{FLAIR + Ours}
& \textbf{29.61} & \textbf{0.803} & 85.50 & \textbf{0.449}
& \textbf{26.81} & \textbf{0.712} & \textbf{52.65} & \textbf{0.346}
& \textbf{32.10} & 0.829 & \textbf{10.80} & \textbf{0.155}
& \textbf{24.49} & \textbf{0.852} & 15.28 & \textbf{0.169} \\
\hline
\multicolumn{17}{c}{\textbf{DIV2K 0.8k}} \\
\hline
FlowChef
& 21.25 & 0.532 & 55.81 & 0.508
& 20.15 & 0.490 & 65.42 & 0.543
& 21.38 & 0.537 & 55.82 & 0.506
& 19.91 & 0.632 & 64.82 & 0.511 \\
FlowDPS
& 21.93 & 0.526 & 43.99 & 0.482
& 20.94 & 0.490 & 53.91 & 0.549
& 20.74 & 0.493 & 60.98 & 0.566
& 21.39 & 0.668 & 41.97 & 0.333 \\
\hline
FlowLPS
& 20.01 & 0.382 & 53.95 & 0.526
& 19.40 & 0.351 & 68.06 & 0.585
& 27.02 & 0.765 & 26.96 & 0.310
& 23.36 & 0.840 & 20.94 & 0.194 \\
\textbf{FlowLPS + Ours}
& 19.76 & \textbf{0.384} & 54.45 & \textbf{0.558}
& \textbf{19.42} & \textbf{0.379} & \textbf{65.76} & \textbf{0.598}
& \textbf{27.04} & \textbf{0.770} & \textbf{25.90} & \textbf{0.302}
& \textbf{24.19} & 0.829 & \textbf{20.86} & 0.213 \\
\hline
FLAIR
& 23.44 & 0.616 & 58.64 & 0.585
& 21.19 & 0.512 & 62.10 & 0.507
& 26.74 & 0.743 & 16.54 & 0.215
& 23.83 & 0.834 & 17.18 & 0.164 \\
\textbf{FLAIR + Ours}
& \textbf{23.51} & \textbf{0.622} & 59.17 & \textbf{0.573}
& \textbf{21.32} & \textbf{0.517} & \textbf{60.36} & \textbf{0.493}
& \textbf{26.87} & \textbf{0.745} & \textbf{15.89} & \textbf{0.207}
& \textbf{24.07} & 0.833 & \textbf{16.61} & 0.167 \\
\hline
\end{tabular}%
\endgroup
\caption{Quantitative results on FFHQ (1k) and DIV2K (0.8k) at
$768\times768$ resolution. All methods use the same number of function
evaluations. ``+ Ours'' augments the corresponding solver with SAS and MPA.}
\label{tab:main-results}
\end{table*}

\begin{figure*}[!h]
    \centering
    \includegraphics[width=1.0\linewidth]{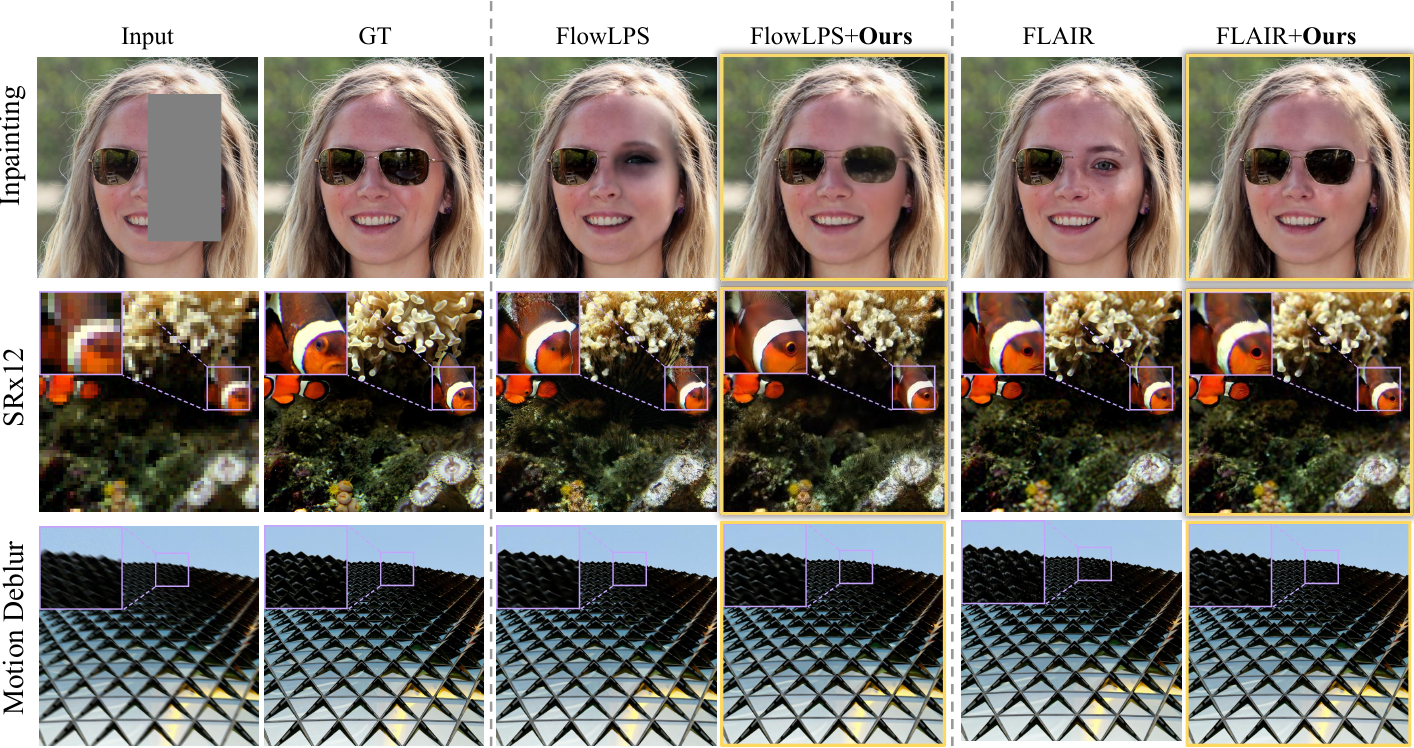}
    \caption{Qualitative comparison across inpainting, $\times12$
    super-resolution, and motion deblurring. Adding our components to FlowLPS
    or FLAIR improves semantic and structural fidelity while maintaining
    consistency with the observed content. Best viewed zoomed in.}
    \label{fig:main-results}
\end{figure*}

\section{Experiments}

\subsection{Setup}
\noindent\textbf{Datasets and Evaluation Metrics.}
We conducted comprehensive evaluations across multiple inverse problems using high-resolution images from two datasets: 1k images from FFHQ~\cite{karras2019stylegan} and 0.8k images from DIV2K~\cite{agustsson2017ntire}. All experiments were performed at $768\times768$ resolution by resizing the original images. Our evaluation employs PSNR and SSIM to quantify pixel-level reconstruction fidelity, complemented by FID and LPIPS for perceptual quality and naturalness.

\noindent\textbf{Baselines.}
We compare with FlowDPS~\cite{kim2025flowdps},
FlowChef~\cite{patel2025flowchef},
FLAIR~\cite{erbach2025flair}, and
FlowLPS~\cite{park2026flowlps}.
The FlowDPS and FlowChef results follow the SD3.0-Medium protocol reported b FlowDPS, whereas FLAIR and FlowLPS use their SD3.5-Medium implementations. Our main controlled comparisons integrate SAS and MPA into FLAIR and FlowLPS, denoted by ``+ Ours''; each augmented solver retains the inference settings
and number of flow-model evaluations of its corresponding baseline.

\noindent\textbf{Problem Setting.}
We run and evaluate all methods at a fixed output resolution of 768×768 pixels. For single image super-resolution, we consider scaling factors of 8× and 12×. The corresponding low-resolution inputs are generated by bicubic downsampling. Motion blur is simulated with a blur kernel of size 61. For box inpainting, we mask large, continuous rectangles that cover approximately one third of the observation. All synthesized observations are corrupted with additive Gaussian noise, with standard deviation $\sigma_v = 0.3\%$.

Following previous work on flow-based inverse problems~\cite{erbach2025flair,park2026flowlps}, we use the fixed prompt ''A high quality photo of a face'' for FFHQ. For DIV2K, we concatenate ''A high quality photo of'' with an image-specific description extracted from the observation using DAPE~\cite{wu2024seesr}.

\subsection{Main Results} 
Table~\ref{tab:main-results} reports the quantitative results. The gains are
most consistent for super-resolution and motion deblurring. Adding our modules
to FLAIR improves most metrics for $\times12$ super-resolution and motion
deblurring on both datasets. FlowLPS also obtains clear improvements on FFHQ
super-resolution.

For inpainting, the clearest quantitative gains occur in fidelity: PSNR improves for both augmented solvers. Figure~\ref{fig:main-results} further shows the qualitative advantage. The baselines often fail to recover instance-specific attributes, including sunglasses and facial markings, whereas our method preserves these structures consistently. The baselines use
the observed content mainly through data-consistency correction, leaving its
structural and semantic cues underused within the generative backbone; their
reconstructions can therefore drift toward a generic semantic basin favored by
the prior. Under the same NFE budget, SAS provides finer temporal resolution
at the noise end for early semantic exploration, while MPA uses
measurement-supported cues to guide spatial feature aggregation toward weakly
constrained regions. Our method also recovers finer textures in
super-resolution and leaves less residual blur in motion deblurring.



\subsection{Ablation Study}

\begin{table}[h]
    \centering
    \small
    \setlength{\tabcolsep}{6pt}
    \begin{tabular}{cccccc}
        \toprule
        SAS & MPA & PSNR $\uparrow$ & SSIM $\uparrow$ & LPIPS $\downarrow$ & FID $\downarrow$ \\
        \midrule
        $\times$     & $\times$     &24.34  &0.435  &0.595  &87.11  \\
        $\checkmark$ & $\times$     &\textbf{26.33}  &\underline{0.574}  &\underline{0.471}  &\textbf{72.94}  \\
        $\times$     & $\checkmark$ &24.41  &0.445  &0.585  &86.27  \\
        $\checkmark$ & $\checkmark$ &\underline{26.19}  &\textbf{0.576}  &\textbf{0.465}  &\underline{73.38}  \\
        \bottomrule
    \end{tabular}
    \caption{Ablation study for $\times12$ super-resolution on FFHQ
    (100 images). SAS and MPA are individually switched on or off. Bold:
    best; underline: second best.}
    \label{tab:ablation_components}
\end{table}
\noindent\textbf{Component Ablation.}
We evaluate the contribution of SAS and MPA in a representative controlled experimental setting based on FlowLPS for FFHQ $\times12$ super-resolution. All variants use the same 100-image subset and NFE budget. Table~\ref{tab:ablation_components} reports the results when SAS and MPA are enabled individually and jointly. SAS contributes most of the quantitative gain and improves all four metrics over the base schedule. MPA alone produces smaller changes in these aggregate metrics, consistent with its main role in preserving semantic and structural details.

Although SAS provides the largest improvements on aggregate metrics, its
semantic recovery remains sensitive to sampling randomness. Allocating more
steps to the noise end provides additional opportunities for early semantic
exploration, but does not guarantee that the trajectory enters the correct
semantic basin. Figure~\ref{fig:seed-stability} examines this limitation on the
same challenging box-inpainting case across different random seeds. The base
solver consistently converges to a generic face favored by the generative
prior, while SAS recovers the complete sunglasses only for some seeds. Changing
the seed can still lead to different semantic outcomes, showing that additional
early exploration alone does not ensure stable semantic recovery.

MPA complements this temporal allocation by acting as a semantic stabilizer.
By strengthening measurement guidance in weakly constrained regions, MPA
allows the missing region to make better use of measurement-supported
structure from the observed content. As shown in
Figure~\ref{fig:seed-stability}, MPA consistently recovers the complete
sunglasses across seeds. This stronger structural constraint may slightly
reduce visual quality in some cases, but it produces markedly more stable
semantic recovery, which is not fully reflected by aggregate metrics.
Combining SAS and MPA retains this semantic stability while improving visual
quality over MPA alone.

\begin{figure}[t]
    \centering
    \includegraphics[width=\linewidth]{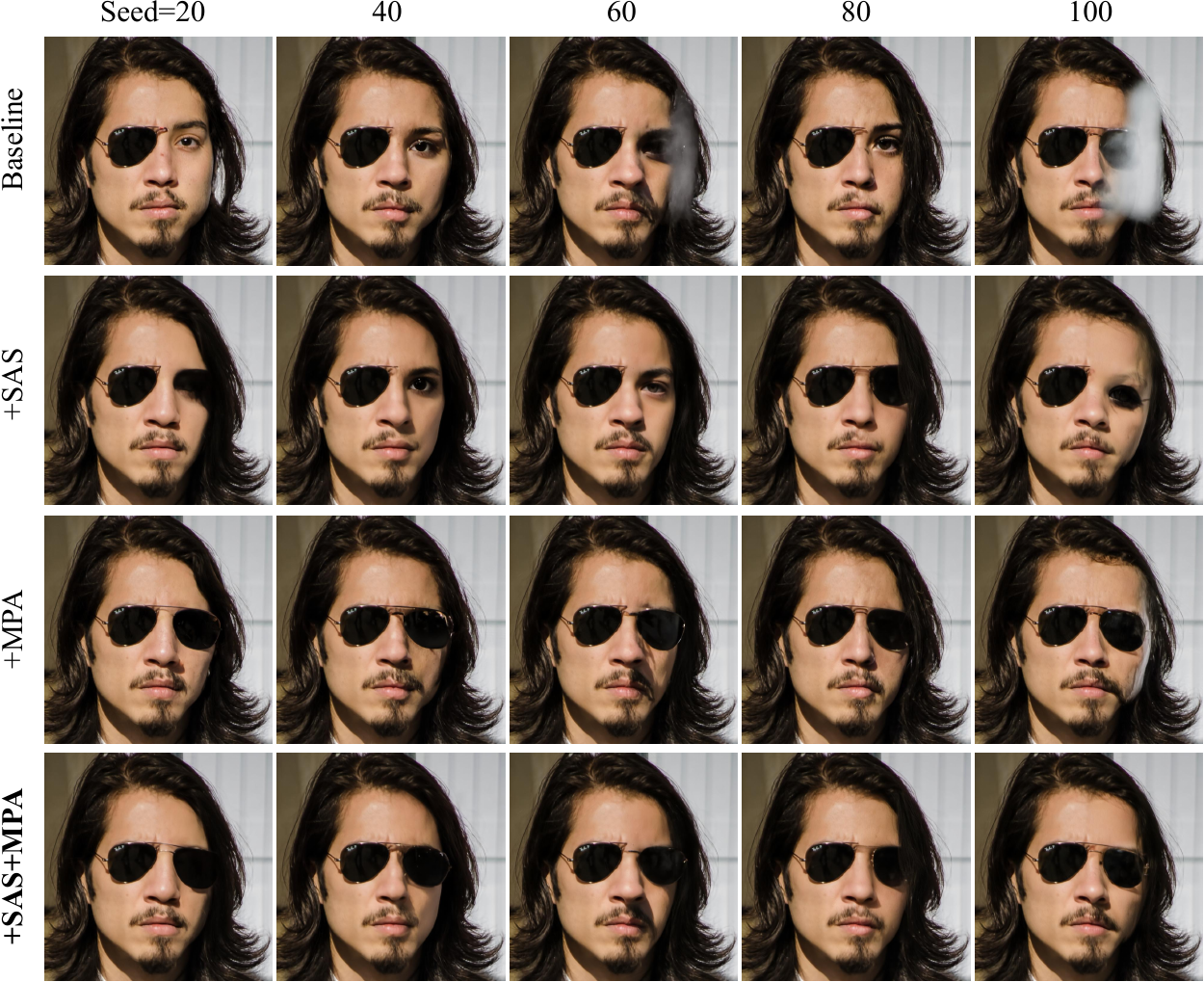}
    \caption{Box-inpainting results across random seeds. Columns correspond to
    different seeds, and rows show the Baseline, $+$SAS, $+$MPA, and
    $+$SAS$+$MPA settings.}
    \label{fig:seed-stability}
\end{figure}

\begin{table}[!h]
    \centering
    \small
    \setlength{\tabcolsep}{5pt}
    \begin{tabular}{lcccc}
        \toprule
        Schedule & PSNR $\uparrow$ & SSIM $\uparrow$ & LPIPS $\downarrow$ & FID $\downarrow$ \\
        \midrule
        Uniform                  & 26.39 & 0.573 & 0.463 & 73.07 \\
        Noise-end dense          & \underline{26.83} & \underline{0.616} & \underline{0.426} & \underline{68.87} \\
        Mid-trajectory dense     & 23.68 & 0.378 & 0.695 & 95.57 \\
        Signal-end dense         & 24.03 & 0.396 & 0.662 & 90.69 \\
        \textbf{SAS (Ours)}        & \textbf{26.89} & \textbf{0.620} & \textbf{0.420} & \textbf{67.99} \\
        \bottomrule
    \end{tabular}
    \caption{Schedule ablation for FlowLPS on FFHQ $8\times$
    super-resolution using 100 images and 50 NFEs.}
    \label{tab:ablation_schedule}
\end{table}

\noindent\textbf{Schedule Ablation.}
To evaluate whether the NFE distribution produced by SAS reflects
operator-specific demand over flow time, we conduct a controlled schedule
ablation with FlowLPS on FFHQ $8\times$ super-resolution. The experiment uses
a 100-image subset and the same 50-NFE budget for all schedules. All other
solver and evaluation settings remain fixed across schedules. We compare SAS
with a uniform schedule and three hand-crafted alternatives. Specifically, for
the hand-crafted alternatives, we divide the trajectory into three equal
segments and assign a larger share of the NFE budget to the noise end, middle,
or clean end, respectively.

Table~\ref{tab:ablation_schedule} shows that SAS performs best on all four
metrics, while the noise-end-dense schedule consistently ranks second. This
result confirms that allocating more steps to early semantic exploration is
beneficial, but also shows that concentrating timesteps in a manually selected
segment is insufficient. By distributing the same NFE budget according to
operator-specific demand, SAS better balances early semantic exploration and
late-stage detail refinement.

\section{Conclusion}
Existing training-free flow-based inverse solvers largely overlook temporal
budget allocation and spatial information propagation under a finite inference
budget. To address these limitations, we introduced Spectrum-Adaptive
Scheduling (SAS) and Measurement-Prioritized Attention (MPA). SAS combines
operator-spectrum weights with logSNR-derived temporal bases to allocate the
fixed timestep budget more effectively between semantic exploration and detail
refinement. MPA uses prior--data conflicts to bias image-token attention in the
DiT backbone, strengthening measurement guidance in weakly constrained
regions. Both components integrate into existing solvers without retraining or
additional inference steps. Experiments across super-resolution, motion
deblurring, and inpainting show consistent improvements in restoration quality
and preservation of instance-specific semantic and structural details. These
results highlight the value of spatio-temporal information allocation under
finite inference budgets.

\bibliography{aaai2027}

\clearpage
\def\ARXIVCOMBINED{}
\ifdefined\ARXIVCOMBINED
\else
\documentclass[letterpaper]{article} 
\usepackage[submission]{aaai2027} 
\usepackage[hyphens]{url} 
\usepackage{graphicx} 
\urlstyle{rm} 
\def\UrlFont{\rm} 
\usepackage{natbib} 
\usepackage{caption} 
\frenchspacing 

\usepackage{amsmath}
\usepackage{amssymb}
\usepackage{amsthm}
\usepackage{booktabs}

\newtheorem{proposition}{Proposition}
\newtheorem{lemma}{Lemma}
\newtheorem{corollary}{Corollary}
\newtheorem{remark}{Remark}

\pdfinfo{
/TemplateVersion (2027.1)
}

\title{Supplementary Material for\\
Making Every Step Count: Spatio-Temporal Information Allocation for Imaging Inverse Problems}
\author{Anonymous Submission}
\affiliations{}

\begin{document}

\maketitle
\fi

\appendix
\setcounter{secnumdepth}{1}

\section{Overview}

This supplement provides concise analysis and implementation details for SAS
and MPA. Sections~\ref{sec:sas-analysis} and~\ref{sec:mpa-analysis} relate
operator spectra and trajectory sensitivity to temporal allocation and
analyze conflict-guided attention, respectively. Sections~\ref{sec:implementation}
and~\ref{sec:add-exp} document solver settings and additional experiments.

\section{Additional Analysis of Spectrum-Adaptive Scheduling}
\label{sec:sas-analysis}

\subsection{Operator spectrum and data-consistency deficit}

Let $\mathcal A\in\mathbb R^{m\times d}$ have singular values
$a_1\geq\cdots\geq a_r>0=a_{r+1}=\cdots=a_d$ and rank $r$. With
$\bar a_k=a_k/a_1$ (and $\bar a_k=0$ for $k>r$), define
\begin{equation}
    G_{\mathcal A}
    =
    \frac{\mathcal A^\top\mathcal A}{\|\mathcal A\|_2^2}
    =
    V\operatorname{diag}(\bar a_1^2,\ldots,\bar a_d^2)V^\top.
    \label{eq:supp-normalized-gram}
\end{equation}
SAS uses
\begin{equation}
    \begin{aligned}
            \alpha_{\rm miss}=&\frac{d-r}{d},
    \quad
    \alpha_{\rm weak}
    =
    \frac{r-\operatorname{srank}(\mathcal A)}{d},
    \\
    &\operatorname{srank}(\mathcal A)=\sum_{k=1}^{r}\bar a_k^2.
    \end{aligned}
    \label{eq:supp-spectral-weights}
\end{equation}

\begin{proposition}[Exact decomposition of measurement deficit]
\label{prop:spectral-deficit}
The SAS spectral weights are nonnegative, invariant to any nonzero scalar
rescaling of $\mathcal A$, and satisfy
\begin{equation}
    \alpha_{\rm miss}+\alpha_{\rm weak}
    =
    \frac{1}{d}\operatorname{tr}(I-G_{\mathcal A})
    =
    \frac1d\sum_{k=1}^{d}(1-\bar a_k^2).
    \label{eq:mean-curvature-deficit}
\end{equation}
For isotropic $u$ with $\mathbb E[uu^\top]=I/d$,
\begin{equation}
    \mathbb E\!\left[u^\top(I-G_{\mathcal A})u\right]
    =
    \alpha_{\rm miss}+\alpha_{\rm weak}.
    \label{eq:isotropic-deficit}
\end{equation}
\end{proposition}

\begin{proof}
Substituting the two definitions yields the sum in
Equation~\eqref{eq:mean-curvature-deficit}. The isotropic identity follows
from $\mathbb E[u^\top Mu]=\operatorname{tr}(M)/d$. Normalization by $a_1$
establishes scale invariance and nonnegativity.
\end{proof}

For $f(x)=\frac12\|\mathcal A x-\mathcal A x^\star\|_2^2$, a gradient step
of size $\|\mathcal A\|_2^{-2}$ gives
\begin{equation}
    e^+
    =
    (I-G_{\mathcal A})e,
    \qquad
    e_k^+
    =
    (1-\bar a_k^2)e_k
    \quad\text{in the right-singular basis.}
    \label{eq:dc-mode-contraction}
\end{equation}
Missing modes receive no correction, whereas attenuated modes contract
slowly; $\alpha_{\rm miss}$ and $\alpha_{\rm weak}$ separate these effects.

\subsection{LogSNR bases as endpoint sensitivities}

Consider the affine flow path
\begin{equation}
    x_t=a_t x_0+b_t\epsilon,
    \qquad
    v_t=\dot a_t x_0+\dot b_t\epsilon,
    \label{eq:supp-affine-system}
\end{equation}
with $\Delta_t=a_t\dot b_t-\dot a_t b_t\neq0$.

\begin{proposition}[Velocity-error sensitivity of endpoint estimates]
\label{prop:endpoint-sensitivity}
Perturbing the predicted velocity by $\delta v_t$ while holding $x_t$ fixed
gives endpoint errors
\begin{equation}
    \delta\widehat x_{0|t}
    =
    -\frac{b_t}{\Delta_t}\delta v_t,
    \qquad
    \delta\widehat\epsilon_{|t}
    =
    \frac{a_t}{\Delta_t}\delta v_t.
    \label{eq:supp-endpoint-error}
\end{equation}
whose normalized squared sensitivities are
\begin{equation}
    \frac{b_t^2}{a_t^2+b_t^2}
    =
    \operatorname{sigmoid}[-\ell(t)],
    \qquad
    \frac{a_t^2}{a_t^2+b_t^2}
    =
    \operatorname{sigmoid}[\ell(t)],
    \label{eq:supp-normalized-sensitivity}
\end{equation}
where $\ell(t)=\log(a_t^2/b_t^2)$.
\end{proposition}

\begin{proof}
Invert Equation~\eqref{eq:supp-affine-system}; normalization cancels the
shared factor $\Delta_t^{-2}$.
\end{proof}

For the linear path, define the normalized sensitivities
$\psi_{\rm prior}(t)=t^2/[t^2+(1-t)^2]$ and
$\psi_{\rm clean}(t)=(1-t)^2/[t^2+(1-t)^2]$. Coupling them to the spectral
deficits gives
\begin{equation}
    D_{\mathcal A}(t)
    =
    \alpha_{\rm miss}\psi_{\rm prior}(t)
    +
    \alpha_{\rm weak}\psi_{\rm clean}(t).
    \label{eq:supp-demand}
\end{equation}
Missing modes therefore shift demand toward the noise endpoint, whereas weak
modes shift it toward the clean endpoint.

\subsection{Trajectory error and allocation density}

Consider $\dot z(t)=f(t,z(t))$ and variable-step Euler times
$t_{\min}=t_0<\cdots<t_N=t_{\max}$, with $h_i=t_{i+1}-t_i$.

\begin{proposition}[Variable-step trajectory-error bound]
\label{prop:trajectory-error}
Assume $f(t,\cdot)$ is $L$-Lipschitz and
$\|\ddot z(t)\|_2\leq\kappa_i$ on $[t_i,t_{i+1}]$. Starting from the exact
initial state, Euler satisfies
\begin{equation}
    \|e_N\|_2
    \leq
    \frac12
    \exp[L(t_{\max}-t_{\min})]
    \sum_{i=0}^{N-1}\kappa_i h_i^2.
    \label{eq:discrete-global-bound}
\end{equation}
\end{proposition}

\begin{proof}
Taylor's theorem gives local error at most $\kappa_i h_i^2/2$.
Unrolling the Lipschitz recursion and using $1+Lh_i\leq e^{Lh_i}$ proves the
claim.
\end{proof}

For a normalized allocation density $q(t)>0$, locally
$h(t)\simeq[Nq(t)]^{-1}$, and Equation~\eqref{eq:discrete-global-bound}
approaches
\begin{equation}
    \frac1N
    \int_{t_{\min}}^{t_{\max}}
    \frac{\kappa(t)}{q(t)}\,dt.
    \label{eq:continuous-error-surrogate}
\end{equation}
For positive continuous $\kappa(t)$, Cauchy--Schwarz gives the unique
normalized density minimizing this surrogate:
\begin{equation}
    q^\star(t)
    =
    \frac{\sqrt{\kappa(t)}}
         {\int_{t_{\min}}^{t_{\max}}\sqrt{\kappa(u)}\,du}.
    \label{eq:optimal-density}
\end{equation}

SAS uses
\begin{equation}
    q_{\mathcal A,\lambda}(t)
    =
    \frac{1+\lambda D_{\mathcal A}(t)}
         {\int_{t_{\min}}^{t_{\max}}
          [1+\lambda D_{\mathcal A}(u)]\,du}.
    \label{eq:supp-sas-density}
\end{equation}
Modeling $\sqrt{\kappa_{\mathcal A}(t)}$ as
$C[1+\lambda D_{\mathcal A}(t)]$ makes SAS optimal for the bound in
Proposition~\ref{prop:trajectory-error}. Here $D_{\mathcal A}$ is an
operator-aware proxy rather than the exact local error of every host solver.

\subsection{Numerical schedule construction}
\label{sec:schedule-implementation}

We evaluate $q_{\mathcal A,\lambda}$ on an 8192-point grid. For $K=50$, the
cumulative distribution and reverse-time global quantiles are
\begin{equation}
    \begin{aligned}
    &F_{\mathcal A,\lambda}(t)
    =\int_{0.18}^{t}q_{\mathcal A,\lambda}(u)\,du,\\
    &t_j
    =F_{\mathcal A,\lambda}^{-1}
      \left(1-\frac{j}{K+1}\right),\quad j=1,\ldots,K.
    \end{aligned}
    \label{eq:supp-global-quantiles}
\end{equation}
The endpoints are not evaluations, and $\lambda=0$ recovers the uniform
schedule.

\section{Additional Analysis of Measurement-Prioritized Attention}
\label{sec:mpa-analysis}

\subsection{Conflict score used by the implementation}

At solver step $s$, let $\delta_s=z_s^{\rm dc}-z_s^{\rm pri}$ and
$h_s(p)=C^{-1}\sum_c|\delta_{s,c}(p)|$. With threshold $\tau$, scale
$v_{\max}$, and exponent $\gamma$,
\begin{equation}
    \bar c_s(p)
    =
    m_{\rm known}(p)
    \left[
    \operatorname{clip}
    \left(
    \frac{h_s(p)-\tau}{v_{\max}-\tau},
    0,1
    \right)
    \right]^\gamma.
    \label{eq:conflict-gate}
\end{equation}
Average pooling gives the final gate
\begin{equation}
    c_s
    =
    m_{\rm known}\odot
    \operatorname{clip}
    \left(
    \operatorname{AvgPool}_k(\bar c_s),0,1
    \right).
    \label{eq:smoothed-conflict}
\end{equation}
The resulting measurement-responsive gate is area-interpolated to each token
grid.

\subsection{Gated outer-product attention bias}

MPA uses $B_s=g_sc_s^\top$, with $g_s=\mathbf1$ for super-resolution and
deblurring and $g_s=1-m_{\rm known}$ for inpainting. Appending
$\sqrt{\beta_s\sqrt{d_h}}g_s$ and
$\sqrt{\beta_s\sqrt{d_h}}c_s$ to the queries and keys gives
\begin{equation}
    \frac{Q'_sK_s'^\top}{\sqrt{d_h}}
    =
    \frac{Q_sK_s^\top}{\sqrt{d_h}}
    +
    \beta_s g_sc_s^\top.
    \label{eq:exact-logit-bias}
\end{equation}
without materializing $B_s$ or disabling fused attention.

\subsection{MPA as a query-gated exponential tilt}

For original attention probabilities $p_{ij}=\operatorname{softmax}_j(L_{ij})$,
Equation~\eqref{eq:exact-logit-bias} gives
\begin{equation}
    \begin{aligned}
            \widetilde p_{ij}(\beta)
    &=
    \frac{p_{ij}\exp(\beta g_i c_j)}
         {Z_i(\beta)},
    \\
    Z_i(\beta)
    &=
    \sum_kp_{ik}\exp(\beta g_i c_k).
    \end{aligned}
    \label{eq:gated-tilt}
\end{equation}

Taking the ratio for keys $j$ and $k$ gives
\begin{equation}
    \frac{\widetilde p_{ij}(\beta)}
         {\widetilde p_{ik}(\beta)}
    =
    \frac{p_{ij}}{p_{ik}}
    \exp[\beta g_i(c_j-c_k)].
    \label{eq:gated-odds}
\end{equation}
Thus larger $c_j$ increases relative attention when $\beta g_i>0$, whereas
$g_i=0$ leaves the row unchanged. In inpainting, only missing-region queries
prioritize high-conflict observed keys.

Let $\mu_i(\beta)=\mathbb E_{\widetilde p_i(\beta)}[c]$. Direct
differentiation gives
\begin{equation}
    \frac{d\mu_i(\beta)}{d\beta}
    =
    g_i\operatorname{Var}_{\widetilde p_i(\beta)}(c)
    \geq0.
    \label{eq:mean-conflict-monotonic}
\end{equation}
so increasing $\beta$ monotonically raises the expected attended conflict
score for gated queries.

\subsection{Conflict correction in FLAIR}

Using FLAIR as the host solver, MPA records the native posterior correction
\begin{equation}
    \delta_s^{\rm FLAIR}
    =
    z_s^{\rm after\ data}
    -
    z_s^{\rm after\ regularizer}.
    \label{eq:flair-delta}
\end{equation}
between its data and regularizer updates. For
$f(z)=\frac12\|\mathcal Az-y\|_2^2$, a gradient correction gives
\begin{equation}
    \delta
    =
    -\eta\mathcal A^\top(\mathcal Az^{\rm pri}-y)
    \in\operatorname{range}(\mathcal A^\top)
    =
    \operatorname{null}(\mathcal A)^\perp.
    \label{eq:conflict-range}
\end{equation}
For inpainting this correction is supported only on observed coordinates;
Equation~\eqref{eq:conflict-gate} selects this measurement-anchored evidence,
and Equation~\eqref{eq:exact-logit-bias} directs it through attention.

\section{Implementation Details}
\label{sec:implementation}

\subsection{Common evaluation protocol}

All reconstructions are produced at $768\times768$ resolution. FFHQ uses the
fixed prompt ``A high quality photo of a face.'' DIV2K uses an image-specific
description extracted from the observation and prefixed by ``A high quality
photo of,'' following the main paper. Images, captions, measurement
realizations, and seeds are matched across every controlled comparison.
Seeds are 3 for super-resolution and 42 for motion deblurring and inpainting.
The measurement-noise level and metric definitions are identical to those in
the main paper.

The local FLAIR and FlowLPS implementations use Stable Diffusion
3.5-Medium, the TAESD3 lightweight autoencoder, a
classifier-free-guidance scale of 2, and an empty negative prompt. Computation
uses half-precision model weights and bfloat16 inference states. The FlowDPS
and FlowChef protocol below follows the matched Stable Diffusion 3.5-Medium
adaptation reported by FLAIR.

The super-resolution operators use bicubic downsampling with factors 8 and
12. Motion deblurring uses a kernel size of 61. FFHQ box inpainting removes
rows $128{:}640$ and columns $384{:}640$, corresponding to a missing fraction
of $2/9$. DIV2K inpainting uses the same fixed two-rectangle mask for every
method. All methods use the same preprocessed targets and observations.

\subsection{Host-solver configurations}

We integrate SAS and MPA into FLAIR and FlowLPS without changing their
within-step update rules.

\paragraph{FLAIR.}
We use a regularizer SGD learning rate of 1 and 15 data-term steps at every
flow time. The data-term learning rate is 12 for $8\times$ and $12\times$
super-resolution and 0.1 for motion deblurring and inpainting. Data-term
optimization stops when its loss falls below $10^{-4}$ times the number of
measurements. The regularizer scale is 0.5 with the released time-dependent
calibration. The regularizer and data term are optimized sequentially; the
conflict correction in Equation~\eqref{eq:flair-delta} is recorded between
these two states.

\paragraph{FlowLPS.}
For $8\times$ and $12\times$ super-resolution, we use 4 Langevin steps and
11 proximal steps, with initial proximal learning rate 0.5 multiplied by
0.85 every five proximal iterations. Motion deblurring uses 6 Langevin steps,
9 proximal steps, and a fixed proximal learning rate of 0.1. Inpainting uses
4 Langevin steps and 11 proximal steps, with initial proximal learning rate
0.1 multiplied by 0.65 every ten proximal iterations. All tasks use a
Langevin step size of $10^{-4}$, one pCN step, and variance $s_t^2=t$.
The model predicts
\begin{equation}
    \widehat z_{0|t}=z_t-tv_\theta(z_t,t),
    \qquad
    \widehat z_{1|t}=z_t+(1-t)v_\theta(z_t,t).
\end{equation}
The solver applies its Langevin, proximal, and single pCN updates, and
records the difference between the proximal and prior endpoint estimates
after the proximal update.
The next state is
\begin{equation}
    z_{t'}
    =
    (1-t')z_{0|t}^{\rm prox}
    +
    t'z_{1|t}^{\rm refreshed}.
\end{equation}

\subsection{SAS parameters}

All SAS experiments use 50 solver times, $t_{\min}=0.18$, $t_{\max}=1$,
$\lambda=1$, an 8192-point numerical grid, and spectral null threshold
$10^{-10}$. The numerical construction is detailed in
Section~\ref{sec:schedule-implementation}. Super-resolution uses
\begin{equation}
    \alpha_{\rm miss}=1-\frac1{s^2},
    \qquad
    \alpha_{\rm weak}=0
\end{equation}
for scale $s$. Inpainting reads the missing fraction directly from the mask.
The super-resolution expression is the ideal rank-deficit approximation: it
treats the $1/s^2$ retained samples as observed modes and does not assign a
separate weak-mode term to the bicubic prefilter. This avoids making the
schedule depend on padding and boundary conventions while retaining the
dominant null-space fraction.
For motion deblurring, the singular values are the magnitudes of the
two-dimensional discrete Fourier transform of the normalized length-61
kernel at image resolution. This is the circulant spectral approximation
used only to construct the schedule; it does not alter the forward operator
used for evaluation. Each singular value is divided by its maximum before
applying Equation~\eqref{eq:supp-spectral-weights}.

\subsection{MPA parameters}

Table~\ref{tab:mpa-settings} gives all task-dependent settings. Shared
settings are $\gamma=0.7$, no exponential moving average, gate updates after
completed solver steps 2 through 35, all transformer layers, and application
only to the conditional half of classifier-free guidance. A gate computed
after one solver step is consumed by the following transformer evaluation.
The same settings are used for both FLAIR and FlowLPS.

\begin{table}[t]
    \centering
    \small
    \setlength{\tabcolsep}{2.5pt}
    \begin{tabular}{lccccc}
        \toprule
        Task & $\beta$ & $\tau$ & $v_{\max}$ & Pool & Query gate\\
        \midrule
        SR $8\times$/$12\times$ & 2.0 & 0.15 & 1.0 & 17 & all\\
        Motion deblur           & 0.5 & 0.02 & 0.5 & 17 & all\\
        Inpainting              & 4.0 & 0.15 & 1.0 & 9 & missing\\
        \bottomrule
    \end{tabular}
    \caption{Task-dependent MPA parameters. ``Pool'' is the side length of
    the average-pooling kernel in latent space.}
    \label{tab:mpa-settings}
\end{table}

\subsection{Baseline configurations}

For both baselines, we use Stable Diffusion 3.5-Medium and the same task
definitions, observations, captions, and seeds as the host solvers.

\paragraph{FlowDPS.}
We use the standard FlowDPS implementation adapted by FLAIR, with a
classifier-free-guidance scale of 2 and 50 NFEs. The update step size is 15
for inpainting and 10 for the other tasks.

\paragraph{FlowChef.}
We use the FLAIR implementation with a classifier-free-guidance scale of 2
and update step size 1. FlowChef uses 50 NFEs for super-resolution and motion
deblurring and 200 NFEs for inpainting.

\section{Additional Experiments}
\label{sec:add-exp}

\subsection{Operator spectra and schedules}
\label{sec:operator-profiles}

Table~\ref{tab:operator-profiles} reports the normalized spectra and the
resulting SAS coefficients. These quantities depend only on the measurement
operator and require no flow-model evaluation.

\begin{table*}[t]
    \centering
    \small
    \setlength{\tabcolsep}{2.8pt}
    \begin{tabular}{lccccc}
        \toprule
        Operator
        & Normalized squared spectrum
        & $r/d$
        & $\operatorname{srank}(\mathcal A)/d$
        & $\alpha_{\rm miss}$
        & $\alpha_{\rm weak}$\\
        \midrule
        SR $8\times$
        & $1\ (d/64),\ 0\ (63d/64)$
        & 0.01563 & 0.01563 & 0.98438 & 0\\
        SR $12\times$
        & $1\ (d/144),\ 0\ (143d/144)$
        & 0.00694 & 0.00694 & 0.99306 & 0\\
        Motion deblur
        & $|\mathcal Fk_{61}|^2/\|\mathcal Fk_{61}\|_\infty^2$
        & 1.00000 & 0.01639 & 0 & 0.98361\\
        Inpainting (FFHQ)
        & $1\ (7d/9),\ 0\ (2d/9)$
        & 0.77778 & 0.77778 & 0.22222 & 0\\
        \bottomrule
    \end{tabular}
    \caption{Normalized spectral statistics used by SAS at
    $768\times768$ resolution. Super-resolution uses the ideal rank model,
    motion deblurring uses the circulant spectrum of the length-61 kernel,
    and inpainting uses the fixed FFHQ box mask.}
    \label{tab:operator-profiles}
\end{table*}

\paragraph{Super-resolution.}
For scale $s$, SAS uses the ideal rank surrogate
\begin{equation}
    \mathcal A_{{\rm SR},s}
    =\bigl(I_{d/s^2}\otimes h_s^\top\bigr)P_s,
    \qquad h_s=s^{-2}\mathbf 1_{s^2},
\end{equation}
where $P_s$ groups nonoverlapping $s\times s$ patches. Since $h_s^\top$ has
one singular value $1/s$, the normalized spectrum contains $d/s^2$ ones and
$d(1-1/s^2)$ zeros. Hence
$(\alpha_{\rm miss},\alpha_{\rm weak})=(1-1/s^2,0)$.

\paragraph{Motion deblurring.}
Under the circular approximation, the length-61 line-kernel surrogate is
Fourier diagonal:
\begin{equation}
    \mathcal A_{\rm blur}
    =\mathcal F^*\operatorname{diag}(\widehat k_{61})\mathcal F,
    \qquad
    \bar a_\omega^2
    =\frac{|\widehat k_{61}(\omega)|^2}
    {\|\widehat k_{61}\|_\infty^2}.
\end{equation}
All sampled magnitudes are nonzero, so $r=d$; Parseval's identity gives
$\operatorname{srank}(\mathcal A)/d=1/61$. Therefore
$(\alpha_{\rm miss},\alpha_{\rm weak})=(0,60/61)$.

\paragraph{Inpainting.}
If $P_m$ moves the $r$ observed pixels first, the exact decomposition is
\begin{equation}
    \mathcal A_{\rm inp}
    =\begin{bmatrix}I_r&0\end{bmatrix}P_m.
\end{equation}
The spectrum thus has $r$ ones and $d-r$ zeros. The fixed FFHQ box removes
$512\times256/768^2=2/9$ of the pixels, giving
$(\alpha_{\rm miss},\alpha_{\rm weak})=(2/9,0)$.

Accordingly, Figure~\ref{fig:operator-schedules} allocates SR and inpainting
by missing modes, with inpainting closer to uniform because its coefficient
is smaller, while motion deblurring is driven by weak modes.

\begin{figure}[t]
    \centering
    \includegraphics[width=\linewidth]
    {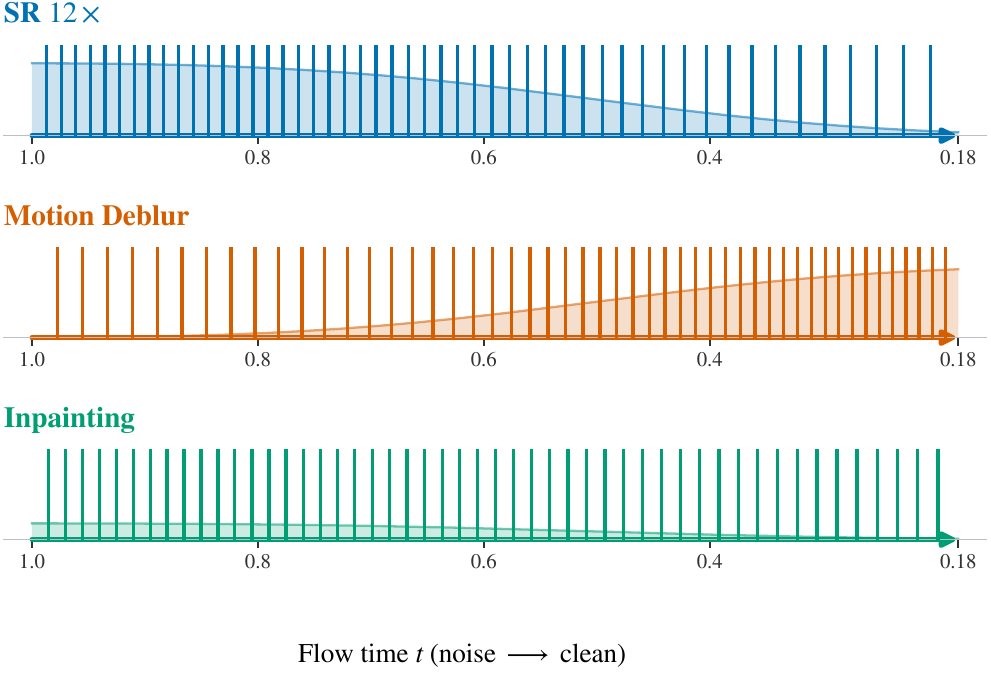}
    \caption{Global SAS schedules for SR $12\times$, motion deblurring, and
    inpainting. Each vertical mark denotes one of the 50 solver times; colors
    distinguish the degradation operators, and the translucent profile shows
    the normalized local density of the displayed times. Flow time proceeds
    from noise to clean.}
    \label{fig:operator-schedules}
\end{figure}

\subsection{Validating the data-consistency deficit}

This diagnostic tests the first link in SAS: operator spectrum
$\rightarrow$ spectral coefficients. It verifies that
$\alpha_{\rm miss}+\alpha_{\rm weak}$ measures the average error retained by
one normalized data-consistency update, giving the coefficients an
operational meaning. It does not evaluate schedule quality or reconstruction
quality; those are tested by the schedule ablation. For an isotropic error
$u$, a larger retained residual means that data consistency leaves more work
to the flow prior. The measured coefficients subsequently weight the two
temporal bases in $D_{\mathcal A}(t)$. We define
\begin{equation}
    R_{\rm lin}(u)
    =
    \frac{u^\top(I-G_{\mathcal A})u}{\|u\|_2^2},
    \qquad
    R_{\rm sq}(u)
    =
    \frac{\|(I-G_{\mathcal A})u\|_2^2}{\|u\|_2^2}.
    \label{eq:dc-retention-ratios}
\end{equation}
The expectations are
$d^{-1}\operatorname{tr}(I-G_{\mathcal A})$ and
$d^{-1}\operatorname{tr}[(I-G_{\mathcal A})^2]$, respectively. We draw 100
Gaussian errors at $768\times768$ resolution and evaluate the update in the
right-singular basis, without a flow model or decoder.

\begin{table*}[t]
    \centering
    \small
    \setlength{\tabcolsep}{3.2pt}
    \begin{tabular}{lcccc}
        \toprule
        Operator
        & \multicolumn{2}{c}{$R_{\rm lin}$}
        & \multicolumn{2}{c}{$R_{\rm sq}$}\\
        \cmidrule(lr){2-3}\cmidrule(lr){4-5}
        & Theory & Measured & Theory & Measured\\
        \midrule
        SR $8\times$
        & 0.98438 & $0.98439{\pm}0.00023$
        & 0.98438 & $0.98439{\pm}0.00023$\\
        Motion deblur
        & 0.98361 & $0.98361{\pm}0.00019$
        & 0.97814 & $0.97815{\pm}0.00023$\\
        Inpainting
        & 0.22222 & $0.22224{\pm}0.00081$
        & 0.22222 & $0.22224{\pm}0.00081$\\
        \bottomrule
    \end{tabular}
    \caption{Theory and numerical measurement of the residual after one
    normalized data-consistency update. Values are means and standard
    deviations over 100 isotropic errors. SR uses the ideal rank model,
    motion deblurring uses the circulant spectrum, and inpainting uses the
    FFHQ box mask.}
    \label{tab:dc-residual-check}
\end{table*}

\subsection{MPA spatial localization and bias strength}

We use one 100-image FFHQ box-inpainting experiment to distinguish spatially
meaningful conflict guidance from a generic increase in attention bias, and
to examine sensitivity to the bias strength $\beta$. All variants use
FLAIR, Global SAS with $\lambda=1$, 50 NFEs, $\tau=0.15$, $v_{\max}=1$,
$\gamma=0.7$, a $9\times9$ pooling kernel, missing-region queries, completed
solver steps 2--35, and all transformer layers.

At the default $\beta=4$, we compare:
\begin{enumerate}
    \item \textbf{No MPA}: the attention bias is disabled.
    \item \textbf{Correct}: use the conflict heatmap in
    Equation~\eqref{eq:smoothed-conflict}.
    \item \textbf{Shuffled}: deterministically permute the heatmap values
    within the known region for each image and active step, preserving its
    histogram and total mass but destroying spatial correspondence.
    \item \textbf{Uniform}: replace all known-region values by their spatial
    mean,
    \begin{equation}
        c_s^{\rm uni}
        =
        m_{\rm known}
        \frac{\sum_jc_{s,j}}{\sum_jm_{{\rm known},j}},
        \label{eq:uniform-heatmap}
    \end{equation}
    preserving support and total heatmap mass while removing localization.
\end{enumerate}
The spatial controls test whether MPA benefits from where the conflict occurs,
rather than merely from adding a positive key-side bias. We separately
evaluate bias-strength sensitivity with the correct heatmap for
$\beta\in\{0,2,3,4,6\}$. The $\beta=0$ endpoint is equivalent to disabling
MPA; all positive values use the same heatmap,
threshold, normalization, pooling, query gate, active steps, and layers.
Figure~\ref{fig:mpa-beta-sensitivity} reports all four restoration metrics as
functions of $\beta$, keeping the strength analysis separate from the spatial
controls.

\begin{figure}[t]
    \centering
    \IfFileExists{Figures/supp_mpa_beta_sensitivity.pdf}{
        \includegraphics[width=\linewidth]
        {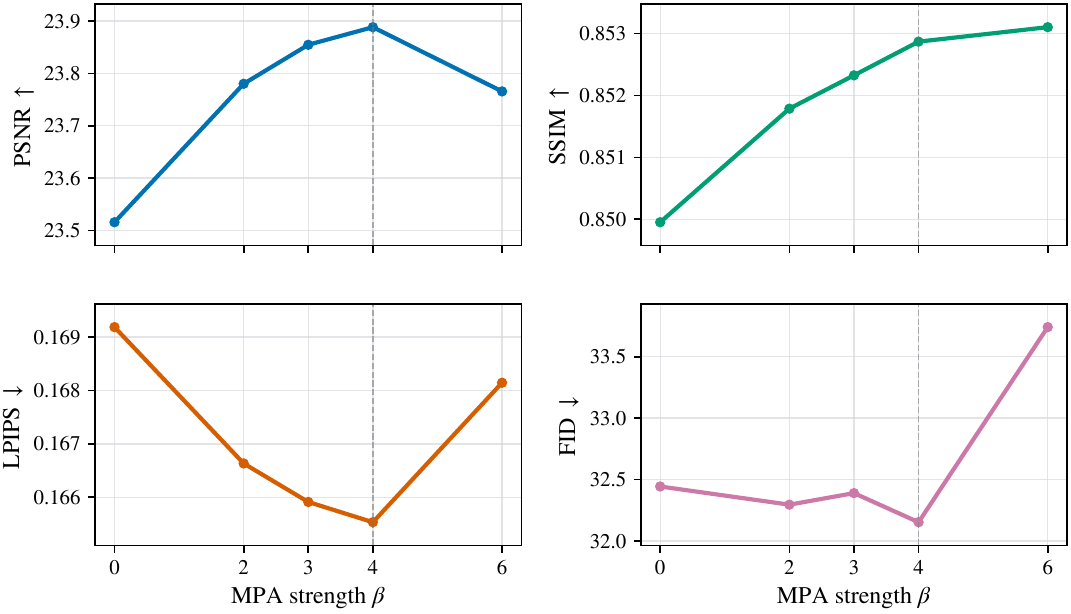}
    }{
        \fbox{\rule{0pt}{1.35in}\rule{0.96\textwidth}{0pt}}
    }
    \caption{MPA bias-strength sensitivity with the correct heatmap on 100
    FFHQ box-inpainting images using FLAIR and Global SAS with $\lambda=1$.
    The four panels report PSNR, SSIM, LPIPS, and FID over
    $\beta\in\{0,2,3,4,6\}$; $\beta=0$ is the no-MPA endpoint.}
    \label{fig:mpa-beta-sensitivity}
\end{figure}

\begin{table}[h]
    \centering
    \small
    \setlength{\tabcolsep}{4.5pt}
    \begin{tabular}{lccccc}
        \toprule
        Heatmap & $\beta$
        & PSNR $\uparrow$ & SSIM $\uparrow$
        & LPIPS $\downarrow$ & FID $\downarrow$\\
        \midrule
        None     & 0 & 23.52 & 0.8500 & 0.1692 & 32.44\\
        Correct  & 4 & \textbf{23.89} & \textbf{0.8529} & 0.1655 & 32.15\\
        Shuffled & 4 & 23.88 & 0.8528 & \textbf{0.1654} & 32.41\\
        Uniform  & 4 & 23.83 & 0.8526 & 0.1660 & \textbf{32.08}\\
        \bottomrule
    \end{tabular}
    \caption{MPA spatial controls for FLAIR on 100 FFHQ box-inpainting
    images. Correct, shuffled, and uniform heatmaps all use $\beta=4$ and
    preserve the remaining MPA settings. Bold denotes the best result in each
    column.}
    \label{tab:mpa-heatmap-controls}
\end{table}

The correct heatmap gives the highest PSNR and SSIM, while the shuffled and
uniform controls are close and give the best LPIPS and FID, respectively.
Thus spatially matched conflict guidance is most evident in the fidelity
metrics under this 100-image protocol. In the strength sweep, $\beta=4$
achieves the best PSNR, LPIPS, and FID; the small further SSIM gain at
$\beta=6$ is accompanied by worse values for the other three metrics.

\begin{figure}[h]
    \centering
    \includegraphics[width=\linewidth]
    {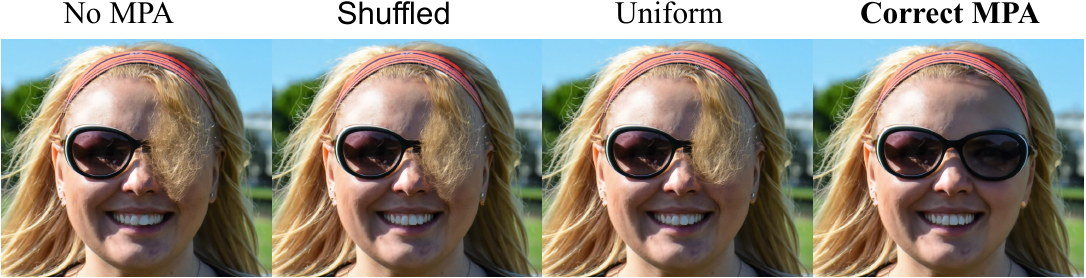}
    \caption{Layout for the MPA spatial-control comparison. The final figure
    should show the key-side heatmap, reconstruction, and enlarged
    missing-region crop for the same input and seed. This placeholder must be
    replaced with experimental outputs before submission.}
    \label{fig:mpa-heatmap-controls}
\end{figure}

\subsection{Runtime, memory, and NFE accounting}

Table~\ref{tab:efficiency-accounting} separates computational overhead from
reconstruction quality. We benchmark FLAIR with batch size one on SR
$8\times$, SR $12\times$, motion deblurring, and box inpainting using an
NVIDIA RTX 4090. For each setting, runtime is averaged over the same 10 FFHQ
images. Model loading, image I/O, and metric computation are excluded; CUDA
is synchronized before and after each reconstruction, and three warm-up
trials are discarded. Peak allocated memory is reset before every measured
image, and we report the maximum of the 10 per-image peaks. All variants use
the same 50 flow-model evaluations.

Across all four tasks, SAS changes runtime by only $-0.02$ to $+0.14$ seconds
per image and leaves peak memory unchanged at the reported precision. The
measurable overhead comes from MPA: Full adds $3.44$--$3.60$ seconds per image
and approximately $0.54$ GiB of peak memory, while preserving the 50-NFE
budget. The small signed differences for SAS are within the run-to-run
variation of this timing protocol.

\subsection{Additional Inpainting Results}

Figure~\ref{fig:additional-inpainting-results} provides further qualitative
evidence for the inpainting behavior discussed in the main paper. Under
matched observations, seeds, and the same 50-NFE budget, the proposed
components more consistently preserve instance-specific semantic and
structural cues across the masked boundary, rather than allowing the
completion to drift toward a generic solution favored by the prior. These
examples reinforce that the improvement extends beyond pixel fidelity to
semantic and structural consistency within the missing region.

\begin{figure*}[t]
    \centering
    \IfFileExists{Figures/supp_inpainting_additional_examples.pdf}{
        \includegraphics[width=\textwidth]
        {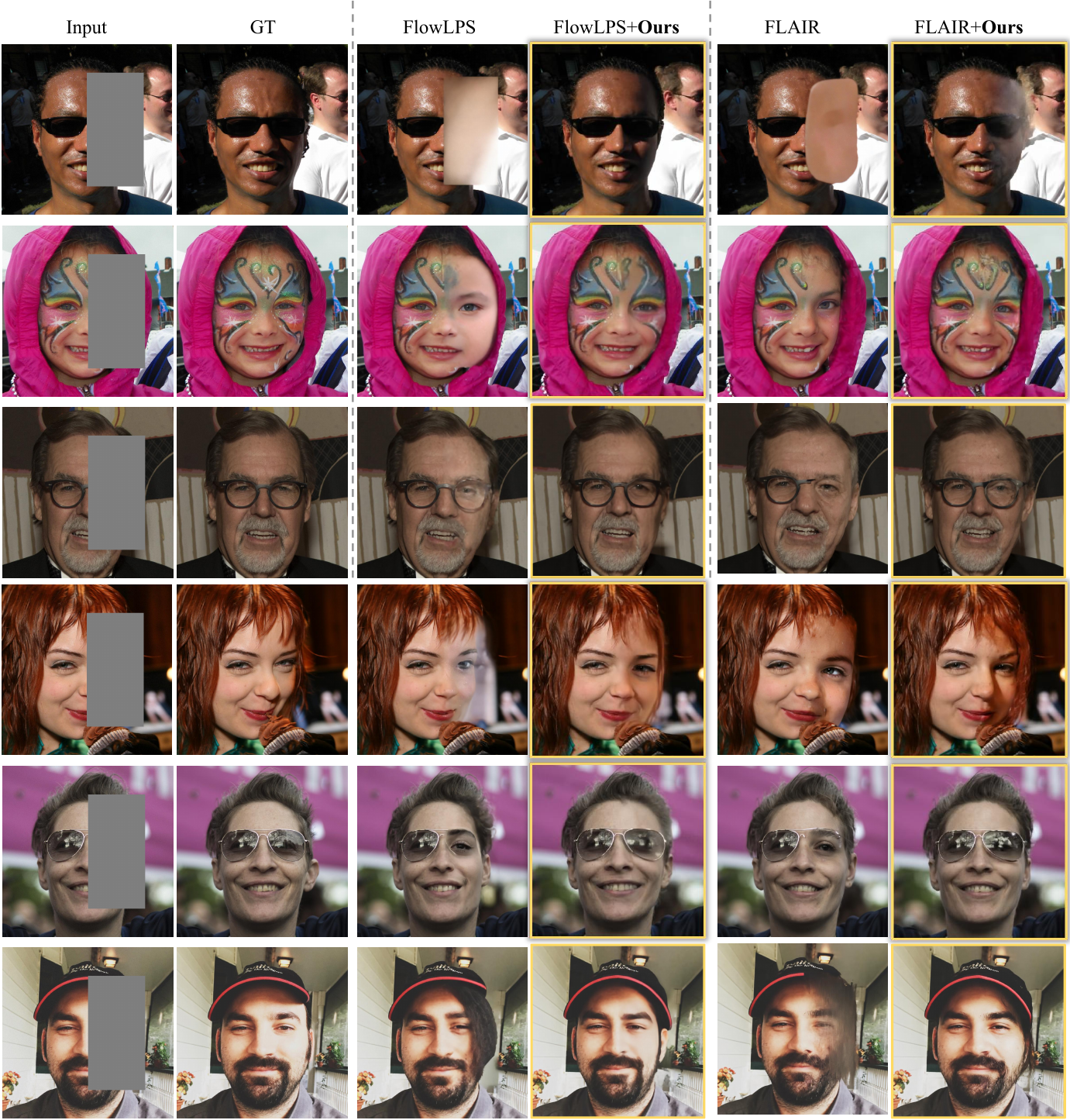}
    }{
        \fbox{\rule{0pt}{2.2in}\rule{0.96\textwidth}{0pt}}
    }
    \caption{Additional qualitative results on FFHQ box inpainting. Under the
    same observations, seeds, and NFE budget, the proposed components more
    consistently preserve instance-specific semantic and structural cues
    across the masked boundary, reinforcing the inpainting observations in the
    main paper.}
    \label{fig:additional-inpainting-results}
\end{figure*}

\subsection{Sensitivity to the SAS strength $\lambda$}

We examine the density-strength parameter using FLAIR on 100 FFHQ
$8\times$ super-resolution images. All variants use the same images,
measurements, seeds, 50-NFE budget, and numerical schedule construction, with
MPA disabled. We vary only $\lambda\in\{0,0.5,1,2,4\}$ in
$1+\lambda D_{\mathcal A}(t)$ and hold all host-solver settings fixed.
Because these are test images, we report sensitivity rather than use them for
hyperparameter selection.

\begin{table}[h]
    \centering
    \small
    \setlength{\tabcolsep}{3.6pt}
    \begin{tabular}{ccccc}
        \toprule
        $\lambda$ & PSNR $\uparrow$ & SSIM $\uparrow$
        & LPIPS $\downarrow$ & FID $\downarrow$\\
        \midrule
        0   & 29.67 & 0.8067 & 0.4481 & \textbf{116.77}\\
        0.5 & 29.74 & 0.8094 & 0.4407 & 116.79\\
        1   & 29.80 & 0.8116 & 0.4348 & 117.28\\
        2   & 29.85 & 0.8134 & 0.4278 & 117.04\\
        4   & \textbf{29.93} & \textbf{0.8163} & \textbf{0.4190} & 117.14\\
        \bottomrule
    \end{tabular}
    \caption{Sensitivity to the SAS strength on 100 FFHQ
    $8\times$ super-resolution images using FLAIR, 50 NFEs, and no MPA.
    Bold denotes the best result in each column.}
    \label{tab:lambda-sensitivity}
\end{table}

Increasing $\lambda$ steadily improves PSNR, SSIM, and LPIPS over the tested
range, while FID varies by only 0.51 without a monotonic trend. These results
show that stronger operator-aware allocation has a consistent but gradual
effect on the per-image metrics rather than introducing an abrupt operating
point.

\subsection{Statistical reporting for 100-image experiments}

For each comparison, we form 100 paired per-image differences using the same
images, measurements, and seeds. We resample these 100 pairs with replacement
10,000 times (seed 2027) and report the observed mean difference with the
2.5th and 97.5th percentiles as a paired 95\% bootstrap confidence interval.
For PSNR and SSIM, the difference is Ours minus baseline; for LPIPS, it is
baseline minus Ours, so positive values always indicate improvement. An
interval excluding zero is treated as significant at the 5\% level. FID is
reported only as a set-level metric because 100 samples do not provide a
reliable per-image significance test.

\section{Computational Overhead}

SAS is computed once per operator and does not evaluate the flow model. Its
spectral calculations are analytic for super-resolution and inpainting and
use one FFT of the blur kernel for motion deblurring. The one-dimensional
schedule construction is performed on an 8192-point grid.

MPA reuses the prior and data-consistent estimates already produced by the
host solver. Constructing the conflict gate costs $O(nd_z)$ for $n$ image
tokens and latent width $d_z$. Equation~\eqref{eq:exact-logit-bias} adds one
feature to each query and key, changing a head's dot-product cost from
$O(n^2d_h)$ to $O[n^2(d_h+1)]$. It does not materialize the outer product,
does not add a transformer call, and introduces no additional flow-model
evaluations. Both components preserve the fixed NFE budget.
\begin{table}[h]
    \centering
    \small
    \setlength{\tabcolsep}{3.5pt}
    \begin{tabular}{lllc}
        \toprule
        Task & Variant
        & Time ($\Delta$Time)
        & GPU memory\\
        \midrule
        SR $8\times$ & Base   & 23.79           & 15.89\\
        SR $8\times$ & $+$SAS & 23.78 {\scriptsize(0.00)}    & 15.89\\
        SR $8\times$ & $+$MPA & 27.23 {\scriptsize($+3.44$)} & 16.43\\
        SR $8\times$ & Full   & 27.24 {\scriptsize($+3.45$)} & 16.43\\
        \midrule
        SR $12\times$ & Base   & 23.81           & 15.89\\
        SR $12\times$ & $+$SAS & 23.78 {\scriptsize($-0.02$)} & 15.89\\
        SR $12\times$ & $+$MPA & 27.20 {\scriptsize($+3.39$)} & 16.43\\
        SR $12\times$ & Full   & 27.25 {\scriptsize($+3.44$)} & 16.43\\
        \midrule
        Motion deblur & Base   & 32.27           & 15.88\\
        Motion deblur & $+$SAS & 32.37 {\scriptsize($+0.10$)} & 15.88\\
        Motion deblur & $+$MPA & 35.89 {\scriptsize($+3.63$)} & 16.42\\
        Motion deblur & Full   & 35.87 {\scriptsize($+3.60$)} & 16.42\\
        \midrule
        Inpainting & Base   & 22.70           & 15.89\\
        Inpainting & $+$SAS & 22.84 {\scriptsize($+0.14$)} & 15.89\\
        Inpainting & $+$MPA & 26.08 {\scriptsize($+3.38$)} & 16.43\\
        Inpainting & Full   & 26.14 {\scriptsize($+3.44$)} & 16.43\\
        \bottomrule
    \end{tabular}
    \caption{Runtime and peak allocated GPU memory for FLAIR on 10 FFHQ
    images per setting. Times are seconds per image; each parenthesized change
    is measured relative to the corresponding Base row. Full denotes SAS+MPA,
    and all variants use 50 NFEs.}
    \label{tab:efficiency-accounting}
\end{table}

\section{Limitations of the Analysis}

The trajectory result in Proposition~\ref{prop:trajectory-error} is a
first-order discretization bound and does not prove that the proposed demand
equals the exact local error of a nonlinear latent inverse solver. The link
to SAS is conditional on an explicit operator-aware difficulty model.
Likewise, the MPA analysis characterizes the attention redistribution but
does not by itself guarantee improvement in the final decoded
image. Their role is to establish the direction, selectivity, and stability
of the intervention; the controlled experiments evaluate whether those
properties improve restoration in practice.

\section{Inpainting Failure Case}

The remaining failures tend to occur near complex mask boundaries where
multiple visually similar contours overlap or intersect. Because the missing
region removes the connections between observed contour fragments, several
continuations may remain plausible, making it difficult to associate each
fragment with its correct counterpart. The full method may consequently merge
or misconnect these contours even when the overall completion remains
semantically plausible. Figure~\ref{fig:inpainting-failure-case} shows a
representative example of this boundary ambiguity.

\begin{figure*}[t]
    \centering
    \IfFileExists{Figures/supp_inpainting_failure_case.pdf}{
        \includegraphics[width=\textwidth]
        {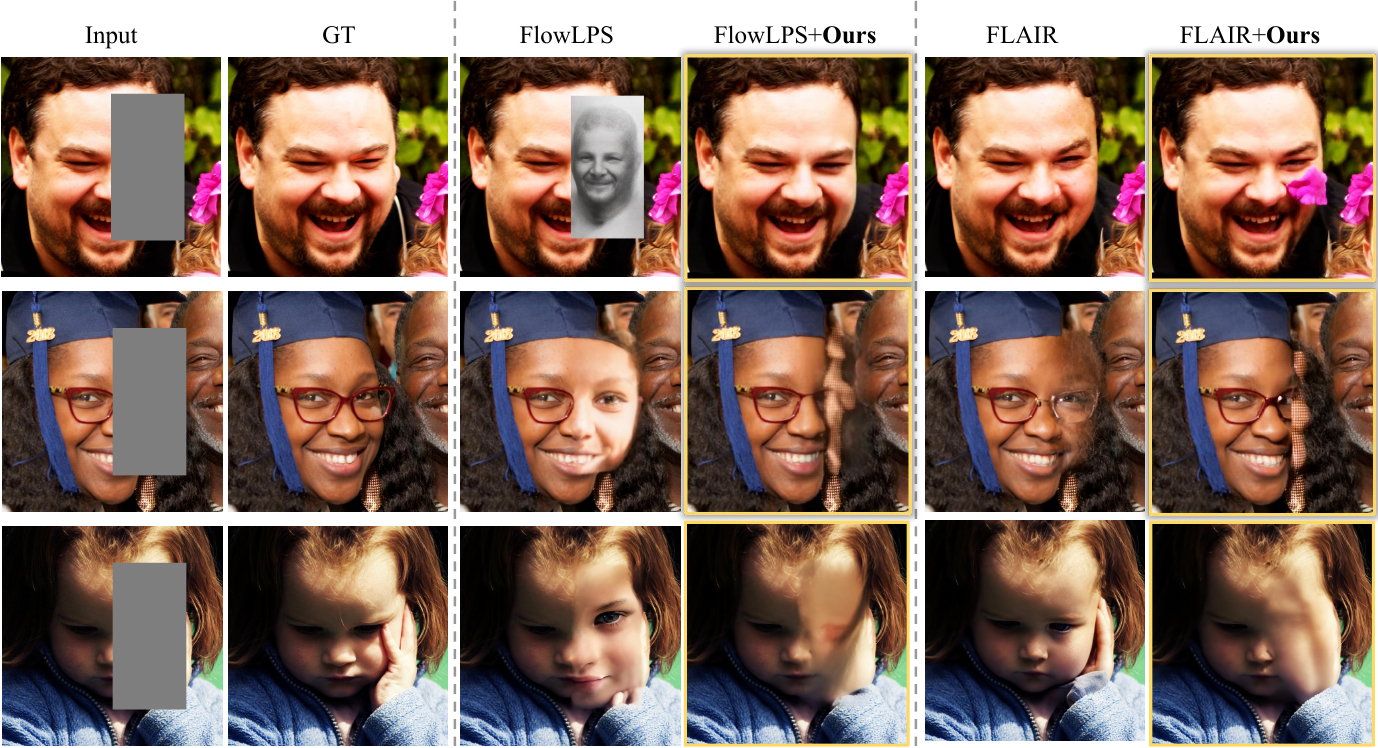}
    }{
        \fbox{\rule{0pt}{1.45in}\rule{0.96\textwidth}{0pt}}
    }
    \caption{Qualitative failure case on FFHQ box inpainting. Columns show
    the target, masked observation, host baseline, and full method for the
    same image and seed. When several similar contours overlap near the masked
    boundary, the full method can produce a plausible but incorrect contour
    association inside the missing region.}
    \label{fig:inpainting-failure-case}
\end{figure*}

\ifdefined\ARXIVCOMBINED
\else
\end{document}
\fi

\end{document}